\documentclass[a4paper, 11pt]{article}

\newif\ifdraft
\draftfalse
\ifdraft
\usepackage{refcheck}
\fi

\usepackage{graphicx, cite}
\usepackage{amsmath}
\usepackage{empheq}
\usepackage{amsfonts}
\usepackage{amssymb}
\usepackage{amsthm}
\usepackage{algorithm}
\usepackage{bbm}
\usepackage[noend]{algpseudocode}

\usepackage{footmisc} 

\usepackage{chngcntr}
\counterwithout{figure}{section}

\allowdisplaybreaks

\newtheorem{theorem}{Theorem}[section]
\newtheorem{problem}[theorem]{Problem}

\newtheorem{lemma}[theorem]{Lemma}
\newtheorem{definition}[theorem]{Definition}

\newtheorem{corollary}[theorem]{Corollary}
\newtheorem{remark}[theorem]{Remark}
\newtheorem{assumption}[theorem]{Assumption}
\usepackage{algorithm}
\floatname{algorithm}{Optimization Problem}

\usepackage{authblk} 
\usepackage{setspace}
\usepackage{xcolor}
\usepackage{xspace}
\usepackage{url} 
\newcommand{\ednote}[3]{\ifdraft {\color{#2} #1: #3} \fi}

\newcommand{\sgnote}[1]{\ednote{SG}{purple}{#1}}
\newcommand{\jtnote}[1]{\ednote{JT}{DarkGreen}{#1}}

\newcommand{\marginnote}[3]{\ifdraft $^{\textcolor{#2}{\mathbf{\dagger}}}$\marginpar{\setstretch{0.43}\textcolor{#2}{\bf\tiny#1: #3}} \fi}
\newcommand{\vkmnote}[1]{\marginnote{VK}{blue}{#1}}
\newcommand{\sgmnote}[1]{\marginnote{SG}{purple}{#1}}
\newcommand{\jtmnote}[1]{\marginnote{JT}{DarkGreen}{#1}}
\newcommand{\akmnote}[1]{\marginnote{AK}{red}{#1}}

\newcommand{\ignore}[1]{}
\definecolor{darkgrey}{rgb}{0.5, 0.5, 0.5}
\definecolor{DarkGreen}{rgb}{0, 0.5, 0}

\usepackage{tikz}
\usetikzlibrary{trees}

\newcommand{\E}{\mathbb{E}}
\newcommand{\R}{\mathbb{R}}
\newcommand{\reals}{\mathbb{R}}
\newcommand{\naturals}{\mathbb{N}}
\newcommand{\eps}{\epsilon}
\renewcommand{\S}{\mathbb{S}}
\newcommand{\B}{\mathcal{B}}
\newcommand{\C}{\mathcal{C}}
\newcommand{\calH}{\mathcal{H}}
\newcommand{\F}{\mathcal{F}}
\newcommand{\W}{\mathcal{W}}
\newcommand{\X}{\mathcal{X}}
\newcommand{\Y}{\mathcal{Y}}
\newcommand{\D}{{\cal D}}
\newcommand{\indicator}{\mathbb{I}}
\newcommand{\lnorm}[1]{\Vert #1 \Vert}

\newcommand{\valpha}{\pmb{\alpha}}
\newcommand{\vv}{\mathbf{v}}
\newcommand{\vx}{\mathbf{x}}
\newcommand{\vw}{\mathbf{w}}
\newcommand{\ve}{\mathbf{e}}
\newcommand{\vK}{\mathbf{K}}

\newcommand{\ReLU}{\mathrm{ReLU}}
\newcommand{\Poly}{\mathcal{P}}
\newcommand{\poly}{\mathrm{poly}}

\newcommand{\falsezero}{\mathcal{L}_{=0}}
\newcommand{\falsejt}{\mathcal{L}}
\newcommand{\ellgtzero}{\mathcal{L}_{>0}}
\newcommand{\hatellgtzero}{\hat{\mathcal{L}}_{>0}}
\newcommand{\ellepszo}{\ell_{\eps\text{-}\mathrm{zo}}}
\newcommand{\Lepszo}{\mathcal{L}_{\eps\text{-}\mathrm{zo}}}
\newcommand{\hatLepszo}{\hat{\mathcal{L}}_{\eps\text{-}\mathrm{zo}}}
\newcommand{\risk}{\mathcal{L}}
\newcommand{\hatrisk}{\hat{\mathcal{L}}}
\newcommand{\calR}{\mathcal{R}}
\newcommand{\clip}{\mathrm{clip}}

\newcommand{\rec}{\sigma_{\mathrm{relu}}}
\newcommand{\sig}{\sigma_{\mathrm{sig}}}
\newcommand{\relu}{\mathrm{relu}}

\newcommand{\MK}{\mathsf{MK}}

\renewcommand{\hat}{\widehat}

\renewcommand{\th}{^{\textit{\scriptsize{th}}}}

\newcommand{\eg}{{e.g.},\xspace}
\newcommand{\ie}{{i.e.},\xspace}
\newcommand{\etal}{{et al.}\xspace}
\newcommand{\etc}{{etc.}\xspace}

\title{Reliably Learning the ReLU in Polynomial Time}

\author[1]{Surbhi Goel}
\author[2]{Varun Kanade}
\author[1]{Adam Klivans}
\author[3]{Justin Thaler}
\affil[1]{University of Texas at Austin}
\affil[2]{University of Oxford and Alan Turing Institute}
\affil[3]{Georgetown University}

\date{}

\begin{document}

\maketitle
\thispagestyle{empty}

\begin{abstract}
	We give the first dimension-efficient algorithms for learning Rectified
	Linear Units (ReLUs), which are functions of the form $\vx \mapsto
	\mathsf{max}(0, ~\vw\cdot \vx)$ with $\vw\in \mathbb{S}^{n-1}$. Our
	algorithm works in the challenging Reliable Agnostic learning model of
	Kalai, Kanade, and Mansour \cite{kkmrel} where the learner is given access
	to a distribution $\D$ on labeled examples but the labeling may be
	arbitrary.  We construct a hypothesis that simultaneously minimizes the
	false-positive rate and the loss
	on inputs given positive labels by $\D$, for any convex, bounded, and Lipschitz loss function.

	The algorithm runs in polynomial-time (in $n$) with respect to {\em any}
	distribution on $\S^{n-1}$ (the unit sphere in $n$ dimensions) and for any
	error parameter $\epsilon = \Omega(1 / \log n)$ (this yields a PTAS for a
	question raised by F. Bach on the complexity of maximizing ReLUs).  These
	results are in contrast to known efficient algorithms for reliably learning
	linear threshold functions, where $\epsilon$ must be $\Omega(1)$ and strong
	assumptions are required on the marginal distribution.  We can compose our
	results to obtain the first set of efficient algorithms for learning
	constant-depth networks of ReLUs.

	Our techniques combine kernel methods and polynomial approximations with a
	``dual-loss'' approach to convex programming.  As a byproduct we obtain a
	number of applications including the first set of efficient algorithms for
	``convex piecewise-linear fitting'' and the first efficient algorithms for
	noisy polynomial reconstruction of low-weight polynomials on the unit
	sphere.
\end{abstract}

\setcounter{page}{0}
\newpage

\section{Introduction}
\label{sec:intro}
Let ${\cal X} = \S^{n-1}$, the set of all unit vectors in $\R^n$, and let
${\cal Y} = [0,1]$.  We define a ReLU (Rectified Linear Unit) to be a function
$f(\vx): {\cal X} \rightarrow {\cal Y}$ equal to $\mathsf{max}(0,~\vw \cdot
\vx)$ where $\vw \in \S^{n-1}$ is a fixed element of $\S^{n-1}$ and $\vw \cdot
\vx$ denotes the standard inner product.\footnote{Throughout this manuscript,
bold lower case variables denote vectors. Unbolded lower case variables denote
real numbers.}  The ReLU is a key building block in the area of deep nets,
where the goal is to construct a network or circuit of ReLUs that ``fits'' a
training set with respect to various measures of loss.  Recently, the ReLU has
become the ``activation function of choice'' for practitioners in deep nets, as
it leads to striking performance in various applications \cite{LeCun15}.

Surprisingly little is known about the computational complexity of
learning even the shallowest of nets: a single ReLU.%
\vkmnote{Should we say PAC-learning ReLUs is easy, although no one has bothered
to write this down?} 
In this work, we provide the first set of efficient algorithms for learning a
ReLU.  The algorithms succeed with respect to {\em any} distribution $\D$
on $\S^{n-1}$, tolerate arbitrary labelings (equivalently viewed as adversarial
noise), and run in polynomial-time for any accuracy parameter $\epsilon =
\Omega(1/\log n)$.  This is in contrast to the problem of learning threshold
functions, i.e., functions of the form $\mathrm{sign}(\vw \cdot \vx)$, where
only computational hardness results are known (unless stronger assumptions are
made on the problem).

Recall the following two fundamental machine-learning problems:

\begin{problem}[Ordinary Least Squares Regression]
Let $\D$ be a distribution on $\S^{n-1} \times [0,1]$.  Given
i.i.d. examples drawn from $\D$, find $\vw \in \S^{n-1}$ that
minimizes $\E_{(\vx,y) \sim \D}[ (\vw \cdot \vx - y)^2]$. 
\end{problem}

\begin{problem}[Agnostically Learning a Threshold Function] \label{prob:ag}
  Let $\D$ be a distribution on $\S^{n-1} \times \{0,1\}$.
  Given i.i.d. examples drawn from $\D$, find $\vw \in \S^{n-1}$
  that approximately minimizes \\ $\Pr_{(\vx,y) \sim \D}[ \mathsf{sign} (\vw \cdot \vx)
  \neq y]$. 
\end{problem}

The term {\em agnostic} above refers to the fact that the labeling on
$\{-1,1\}$ may be {\em arbitrary}.  In this work, we relax the notion
of success to {\em improper learning}, where the learner may output
any polynomial-time computable hypothesis achieving a loss that is 
within $\epsilon$ of the optimal solution from the concept class.

Taken together, these two problems are at the core of many important
techniques from modern Machine Learning and Statistics.  It is
well-known how to efficiently solve ordinary least squares and other
variants of linear regression; we know of multiple polynomial-time
solutions, all extensively used in practice \cite{Rig}.  In contrast,
Problem \ref{prob:ag} is thought to be
computationally intractable due to the many existing hardness results
in the literature \cite{Daniely, KalaiKMS2008, KS09, FGKP09}.

The ReLU is a hybrid function that lies ``in-between'' a linear function and a
threshold function in the following sense:  restricted to inputs $\vx$ such
that $\vw \cdot \vx > 0$, the ReLU is linear, and for inputs $\vx$ such that
$\vw \cdot \vx \leq 0$, the ReLU thresholds the value $\vw \cdot \vx$ and
simply outputs zero. In this sense, we could view the ReLU as a ``one-sided''
threshold function.  Since learning a ReLU has aspects of both linear
regression and threshold learning, it is not straightforward to identify a
notion of loss that captures both of these aspects.

\subsection{Reliably Learning Real-Valued Functions}

We introduce a natural model for learning ReLUs inspired by the Reliable
Agnostic learning model that was introduced by Kalai et al. \cite{kkmrel} in
the context of Boolean functions. The goal will be to minimize both the false
positive rate and a loss function (for example, square-loss) on points the
distribution labels non-zero.  In this work, we give efficient algorithms for
learning a ReLU over the unit sphere with respect to any loss function that
satisfies mild properties (convexity, monotonicity, boundedness, and
Lipschitz-ness).

The Reliable Agnostic model is motivated by the
Neyman-Pearson criteria, and is intended to capture settings in which
false positive errors are more costly than false negative errors
(e.g., spam detection) or vice versa.  We observe that the asymmetric
manner in which the Reliable Agnostic model \cite{kkmrel} treats
different types of errors naturally corresponds to the one-sided
nature of a ReLU. In particular, there may be settings in which
mistakenly predicting a positive value instead of zero carries a high
cost.

As a concrete example, imagine that inputs are comments on an online news
article.  Suppose that each comment is assigned a numerical score of quality or
appropriateness, where the true scoring function is reasonably modeled by a
linear function  of the features of the comment. The newspaper wants to
implement an automated system in which comments are either a) rejected outright
if the score is below a threshold or b) posted in order of score, possibly
after undergoing human review.\footnote{For example, The New York Times
recently announced that they are moving to a hybrid comment moderation system
that combines human and automated review \cite{nytimes}.} In this situation, it
may be costlier to post (or subject to human review) a low-quality or
inappropriate comment than it is to automatically reject a comment that is
slightly above the threshold for posting.

More formally, for a function $h$ and distribution $\D$ over
$\mathbb{R}^n \times [0, 1]$ define the following losses
\begin{align*}
	\falsezero(h) &= \underset{(\vx,y) \sim \D}{\Pr}[h(\vx) \neq 0 \land\ y=0]  \\
	\ellgtzero(h) &= \underset{(\vx,y) \sim \D}{\mathbb{E}}[\ell(h(\vx),
	y) \cdot \mathbb{I}(y > 0)]. 
\end{align*}

Here, $\ell$ is a desired loss function, and $\mathbb{I}(y > 0)$ equals 0 if $y
\leq 0$ and 1 otherwise.  These two quantities are respectively the
false-positive rate and the expected loss (under $\ell$) on examples for which
the true label $y$ is positive.%
\footnote{We restrict $\Y = [0, 1]$ as it is a natural setting for the case of
ReLUs. However, our results can easily be extended to larger ranges.}
\vkmnote{%
This will be the right place to discuss our choice of $[0, 1]$ vs $[-1, 1]$ a
bit. As I pushed for it, I'll add this discussion once more important things
are in place.} 
\akmnote{I would suggest simply a footnote saying that we can handle the range
[-1,1] or a parenthetical saying we can handle larger ranges in general}

Let $\mathcal{C}$ be a class of functions mapping $\S^{n-1}$ to $[0, 1]$ (e.g.,
$\mathcal{C}$ may be the class of all ReLUs).  Let $\mathcal{C}^+ = \lbrace c
\in \mathcal{C}~|~ \falsezero(c)=0\}.$ We say ${\cal C}$ is \emph{reliably
learnable} if there exists a learning algorithm ${\cal A}$ that (with high
probability) outputs a hypothesis that 1) has at most $\epsilon$ false positive
rate and 2) on points with positive labels, has expected loss that is within
$\epsilon$ of the best $c$ from ${\cal C^{+}}$.  That is, the hypothesis must
be both {\em reliable} and competitive with the optimal classifier from the
class $\mathcal{C}^+$ ({\em agnostic}). 

\subsection{Our Contributions}
\label{ssec:contrib}
\label{sec:lossdiscussion}
\label{sec:relutheoremsec}

We can now state our main theorem giving a poly-time algorithm (in $n$, the
dimension) for reliably learning any ReLU.

All of our results hold for loss functions $\ell$ that satisfy convexity,
monotonicity, boundedness, and Lipschitz-ness. For brevity, we avoid making
these requirements explicit in the theorem statements of this introduction, and
we omit the dependence of the runtime on the failure probability $\delta$ of
the algorithm or on the boundedness and Lipschitz parameters of the loss function. 
All theorem statements in subsequent sections do state
explicitly to what class of loss functions they apply, as well as the runtime
dependence on these additional parameters.

\begin{theorem} \label{thm:relu}
  Let $\mathcal{C}=\{\vx \mapsto \mathsf{max}(0, \vw \cdot \vx) \colon
	\|\vw\|_2 \leq 1\}$ be the class of ReLUs with weight vectors $\vw$
	satisfying $\|\vw\|_2 \leq 1$.  There exists a learning algorithm ${\cal A}$
	that reliably learns ${\cal C}$ in time $2^{O(1/\epsilon)} \cdot n^{O(1)}$. 
\end{theorem}

\begin{remark}
	We can obtain the same complexity bounds for learning ReLUs in the standard
	agnostic model with respect to the same class of loss functions.  This
	yields a PTAS (polynomial-time approximation scheme) for an optimization
	problem regarding ReLUs posed by Bach \cite{Bac:2014}.  See Section
	\ref{bach} for details.
\end{remark}

For the problem of learning threshold functions, all known
polynomial-time algorithms require strong assumptions on the marginal
distribution (e.g., Gaussian \cite{KalaiKMS2008} or large-margin
\cite{SSSS}).  In contrast, for ReLUs, we succeed with respect to {\em
  any} distribution on $\S^{n-1}$.  We leave open the problem of
improving the dependence of Theorem \ref{thm:relu} on $\epsilon$. We
note that for the problem of learning threshold functions---even
assuming the marginal distribution is Gaussian---the run-time
complexity must be at least $n^{\Omega(\log 1/\epsilon)}$ under the
widely believed assumption that learning sparse parities is
hard~\cite{KK:2014}.  Further, the best {\em known} algorithms for agnostically
learning threshold functions with respect to Gaussians run in time
$n^{O(1/\epsilon^2)}$ \cite{KalaiKMS2008,DKN10}. Contrast this to our result
for learning ReLUs, where we
give polynomial-time algorithms even for $\epsilon$ as small as
$1/\log n$.

We can compose our results to obtain efficient algorithms for
small-depth networks of ReLUs. For brevity, here we state results only for
linear combinations of ReLUs (which are often called \emph{depth-two} networks
of ReLUs, see, e.g., \cite{elden}).  Formal results for other types of networks
can be found in Section \ref{sec:networkrelu}.

\begin{theorem}
Let ${\cal C}$ be a depth-2 network of ReLUs with $k$ hidden units. Then
${\cal C}$ is reliably learnable in time $2^{O(\sqrt{k}/\eps)} \cdot
n^{O(1)}$. 
\end{theorem}

The above results are  perhaps surprising in light of the hardness result due
to Livni et al. \cite{LivniSS14} who showed that for
${\cal X} = \{0,1\}^n$, learning the difference of even two ReLUs is
as hard as learning a threshold function.

We also obtain results for {\em noisy polynomial reconstruction} on the
sphere (equivalently, agnostically learning a polynomial) with respect
to a large class of loss functions:

\begin{theorem} \label{thm:poly}
  Let ${\cal C}$ be the class of polynomials $p \colon \S^{n-1} \to [-1, 1]$ in
	$n$ variables such that that the total degree of $p$ is at most $d$, and the
	sum of squares of coefficients of $p$ (in the standard monomial basis) is at
	most $B$.  Then ${\cal C}$ is agnostically learnable
	under any (unknown) distribution over $\mathbb{S}^{n-1} \times [-1, 1]$ in
	time $\mathsf{poly}(n, d, B, 1/\epsilon)$.
\end{theorem}

Andoni et al. \cite{AndoniPV014} were the first to give efficient
algorithms for noisy polynomial reconstruction over non-Boolean
domains.  In particular, they gave algorithms that succeed on the unit
cube but require an underlying product distribution and do not work in
the agnostic setting (they also run in time exponential in the degree
$d$). 

At a high level the proofs of both Theorem \ref{thm:relu} and
\ref{thm:poly} follow the same outline, but we do not know how to
obtain one from the other. 

\subsection{Applications to Convex Piecewise Regression}

We establish a novel connection between learning networks of ReLUs and a broad
class of piecewise-linear regression problems studied in machine learning and
optimization.  The following problem was defined by Boyd and Magnani \cite{Mag}
as a generalization of the well-known MARS (multivariate adaptive regression
splines) framework due to Friedman \cite{MARS}:

\begin{problem} [Convex Piecewise-Linear Regression: Max $k$-Affine]
	\label{piecewise-def}
	Let ${\cal C}$ be the class of functions of the form $f(x) =
	\mathsf{max}(\vw_{1} \cdot \vx, \ldots, \vw_{k} \cdot \vx)$ with $\vw_{1},
	\ldots, \vw_{k} \in \mathbb{S}^{n-1}$ mapping $\S^{n-1}$ to $\R$.  Let
	$\D$ be an (unknown) distribution on $\S^{n-1} \times [-1,1]$.  Given
	i.i.d. examples drawn from $\D$, output $h$ such that $\E_{(\vx,y)
	\sim \D}[ (h(\vx) - y)^2] \leq  \min_{c \in \cal{C}} \E_{(\vx,y) \sim
	\D}[ (c(\vx) - y)^2] + \eps$ .   
\end{problem}

Applying our learnability results for networks of ReLUs, we obtain the first
polynomial-time algorithms for solving the above {\em max-$k$-affine}
regression problem and the {\em sum of max-$2$-affine} regression problem when
$k = O(1)$.  Boyd and Magnani specifically highlight the case of $k = O(1)$ and
provide a variety of heuristics;  we obtain the first provably efficient
results.

\begin{theorem}
	There is an algorithm ${\cal A}$ for solving the convex piecewise-linear
	fitting problem (cf. Definition \ref{piecewise-def}) in time
	$2^{O\left((k/\epsilon)^{{\log k}}\right)} \cdot n^{O(1)}$. 
\end{theorem}

We can also use our results for learning networks of ReLUs to learn
the so-called ``leaky ReLUs'' and ``parameterized'' ReLUs (PReLUs);
see Section \ref{sec:prelu} for details.  We obtain these results by
composing various ``ReLU gadgets,'' i.e., constant-depth networks of
ReLUs with a small number of bounded-weight hidden units.

\subsection{Hardness}

We also prove the first hardness results for learning a single ReLU
via simple reductions to the problem of learning sparse parities with
noise. These results highlight the difference between learning Boolean
and real-valued functions and justify our focus on (1) input distributions
over $\S^{n-1}$ and (2) learning problems that are {\em not} scale
invariant (for example, learning a linear threshold function over the
Boolean domain is equivalent to learning over $\S^{n-1}$ in the
distribution-free setting).

\begin{theorem}
  Let ${\cal C}$ be the class of ReLUs over the domain
  ${\cal X} = \{0,1\}^n$.  Then any algorithm for reliably learning
  ${\cal C}$ in time $g(\epsilon) \cdot \mathsf{poly}(n)$ for any
  function $g$ will give a polynomial time algorithm for learning $\omega(1)$-sparse
  parities with noise (for any $\epsilon = O(1)$).
\end{theorem}

Efficiently learning sparse parities (of any superconstant length)
with noise is considered one the most challenging problems in
theoretical computer science.%
\vkmnote{Should we may be change this to computational learning theory, lest we
compare this to P $\neq$ NP and so on?}

\subsection{Techniques and Related Work}
\label{ssesc:related}

We give a high-level overview of our proof.  Let ${\cal C}$ be the class of all
ReLUs, and let $S = \{(\vx_1,y_1), \ldots, (\vx_m, y_m)\}$ be a training set of
examples drawn i.i.d. from some arbitrary distribution $\D$ on
$\S^{n-1} \times [-1,1]$.  To obtain our main result for reliably learning a
single ReLU (cf. Theorem \ref{thm:relu}), our starting point is Optimization
Problem~\ref{alg:optprob2} below.

\begin{algorithm}[H]
	\caption{\label{alg:optprob2}}
	\begin{align*}
		\underset{\vw}{\text{minimize}} \quad\quad &\sum_{i: y_i > 0} \ell(y_i,
		\mathsf{max}(0, \vw \cdot \vx_i)) \\
		\text{subject to} \quad\quad &\mathsf{max}(0, \vw \cdot \vx_i) = 0 \quad
		\text{for all $i$ such that } y_i = 0 \\
		&~\quad\quad\quad~\lnorm{\vw}_2 \leq 1
	\end{align*}
\end{algorithm}

In Optimization Problem~\ref{alg:optprob2}, $\ell$ denotes the loss function
used to define $\ellgtzero$. Using standard generalization error arguments, it
is possible to show that (for reasonable choices of $\ell$) if $\vw$ is an
optimal solution to Optimization Problem~\ref{alg:optprob2} when run on a
polynomial size sample $(\vx_1, y_1), \dots, (\vx_m, y_m)$ drawn from $\D$, then
it is sufficient to output the hypothesis $h(\vx):= \mathsf{max}(0, \vw \cdot
\vx)$.  Unfortunately Optimization Problem \ref{alg:optprob2} is not convex in
$\vw$, and hence it may not be possible to find an optimal solution in
polynomial time.  Instead, we will give an efficient approximate solution that
will suffice for reliable learning. 

Our starting point will be to prove the existence of low-degree,
low-weight polynomial approximators for every $c \in {\cal C}$.  The
polynomial method has a well established history in computational
learning theory (e.g., Kalai et al. \cite{KalaiKMS2008} for
agnostically learning halfspaces under distributional assumptions),
and we can apply classical techniques from approximation theory and
recent work due to Sherstov \cite{She:2012} to construct low-weight,
low-degree approximators for any ReLU.

We can then relax Optimization Problem 1 to the space of low-weight
polynomials and follow the approach of Shalev-Shwartz et
al. \cite{SSSS} who used tools from Reproducing Kernel Hilbert Spaces
(RKHS) to learn low-weight polynomials efficiently (Shalev-Shwartz et
al. focused on a relaxation of the 0/1 loss for halfspaces).

The main challenge is to obtain reliability; i.e., to simultaneously
minimize the false-positive rate and the loss dictated by the
objective function.  To do this we take a ``dual-loss'' approach and
carefully construct two loss functions that will both be minimized
with high probability.  Proving that these losses generalize for a
large class of objective functions is subtle and requires ``clipping''
in order to apply the appropriate Rademacher bound.  Our final output
hypothesis is $\mathsf{max}(0, h)$ where $h$ is a ``clipped'' version of the
optimal low-weight, low-degree polynomial on the training data,
appropriately kernelized.

Our learning algorithms for networks of ReLUs are obtained by generalizing a
composition technique due to Zhang et al. \cite{Zhang}, who considered networks
of ``smooth'' activation functions computed by power series (we
discuss this more in Section \ref{sec:networkrelu}).  Using a sequence
of ``gadget'' reductions, we then show that even small-size networks of ReLUs are
surprisingly powerful, yielding the first set of provably efficient algorithms for a
variety of piecewise-linear regression problems in high dimension. \\

\noindent{\em {\bf Note:}} A recent manuscript appearing on the Arxiv
due to R. Arora et al. \cite{rarora} considers the complexity of training depth-$2$ networks of
ReLUs with $k$ hidden units on a sample of size $m$ but {\em when the
  dimension $n = 1$.}   They give a proper learning algorithm that runs in time
exponential in $k$ and $m$.  These concept classes, however, can be improperly learned
in time polynomial in $k$ and $m$ using a straightforward reduction to
piecewise linear regression on the real line.

\section{Preliminaries}
\label{sec:prelims}
\subsection{Notation}

The input space is denoted by $\X$ and the output space by $\Y$. In most of
this paper, we consider settings in which $\X = \S^{n-1}$, the unit sphere in
$\reals^n$,\footnote{All of our algorithms would also work under arbitrary
distributions over the unit \emph{ball}.}  and $\Y$ is either $[0, 1]$ or $[-1,
1]$. Let $\B_n(0, r)$ denote the origin centered ball of radius $r$ in
$\reals^n$. 

We denote vectors by boldface lowercase letters such as $\vw$ or $\vx$, and
$\vw \cdot \vx$ denotes the standard scalar (dot) product. By $\lnorm{\vw}$ we
denote the standard $\ell_2$ (i.e., Euclidean) norm of the vector $\vw$; when
necessary we will use subscripts to indicate other norms.  If $f \colon
\S^{n-1} \to \mathbb{R}$ is a real-valued function over the unit sphere, we say
that a multivariate polynomial $p$ is an $\eps$-approximation to $f$ if
$|p(\vx)-f(\vx)| \leq \eps$ for all $\vx \in \S^{n-1}$. 
For a natural number $n \in \naturals$, $[n] = \{0, 1, \ldots, n\}$.

\subsection{Concept Classes}

Neural networks are composed of units---each unit has some $\vx \in \reals^n$
as input (for some value of $n$, and  $\vx$ may consist of outputs of other
units) and the output is typically a linear function composed with a non-linear
\emph{activation} function, \ie the output of a unit is of the form $f(\vw
\cdot \vx)$, where $\vw \in \reals^n$ and $f : \reals \rightarrow \reals$.

\begin{definition}[Rectifier]
	The rectifier (denoted by $\rec$) is an activation function defined as
	$\rec(x) = \mathsf{max} (0,x)$. 
\end{definition}

\begin{definition}[$\ReLU(n, W)$]
	For $\vw \in \reals^n$, let $\relu_{\vw} : \reals^n \rightarrow \reals$
	denote the function $\relu_{\vw}(\vx) = \mathsf{max}(0, \vw \cdot \vx)$.  Let $W \in
	\reals^+$;\vkmnote{Be consistent with $\reals^+$ vs $\reals_+$. Personally, I prefer $\reals^+$ and have tried to use everywhere.} we denote
	by $\ReLU(n, W)$ the class of rectified linear units defined by $\{
		\relu_{\vw} ~|~ \vw \in \B_n(0, W) \}$.
\end{definition}

Our results on reliable learning focus on the class $\ReLU(n, 1)$. We define
networks of $\ReLU$s in Section~\ref{sec:networkrelu}, where we also present
results on agnostic learning and reliable learning of networks of $\ReLU$s.

\begin{definition}[$\Poly(n, d, B)$]\label{defn:polyndb}
	Let $B \in \reals^+$, $n, d \in \naturals$. We denote by $\Poly(n, d, B)$
	the class of $n$-variate polynomials $p$ of total degree at most $d$ such
	that the sum of the squares of the coefficients of $p$ in the standard
	monomial basis is bounded by $B$.
\end{definition}

\subsection{Learning Models}
\label{sec:defs}
We consider two learning models in this paper. The first is the standard agnostic
learning model~\cite{KSS:1994, Haussler:1992} and the second is a
generalization of the reliable agnostic learning framework~\cite{kkmrel}. We
describe these models briefly; the reader may refer to the original
articles for further details.

\begin{definition}[Agnostic Learning~\cite{KSS:1994, Haussler:1992}] 
	We say that a concept class $\C \subseteq \Y^{\X}$ is agnostically learnable
	with respect to loss function $\ell : \Y^\prime \times \Y \rightarrow
	\reals^+$ (where $\Y \subseteq \Y^\prime$), if for every $\delta, \epsilon >
	0$ there exists a learning algorithm $\mathcal{A}$ that for every
	distribution $\D$ over $\X \times \Y$ satisfies the following. Given access
	to examples drawn from $\D$, $\mathcal{A}$ outputs a hypothesis $h : \X
	\rightarrow \Y^\prime$, such that with probability at least $1 - \delta$,
	\begin{equation}
		\label{agnosticdefeq}  \E_{(\vx, y) \sim \D} [\ell(h(\vx), y)] \leq
		\min_{c \in C} \E_{(\vx, y) \sim \D} [\ell(c(\vx), y)] + \epsilon.
	\end{equation}
	Furthermore, if $\X \subseteq \reals^n$ and $s$ is a parameter that captures
	the representation complexity (i.e., description length) of concepts $c$ in
	$\C$, we say that $\C$ is \emph{efficiently agnostically learnable to error
	$\epsilon$} if $\mathcal{A}$ can output an $h$ satisfying Equation
	\eqref{agnosticdefeq} with running time polynomial in $n$, $s$, and
	$1/\delta$.\footnote{The accuracy parameter $\epsilon$ is purposely omitted
	from the definition of efficiency; in our results we will explicitly state
	the dependence on $\epsilon$ and for what ranges of $\epsilon$ the running
	time remains polynomial in the remaining parameters.\label{ftn:efflearn}}
\end{definition}
	
Next, we formally describe our extension of the reliable agnostic learning
model introduced by Kalai~\etal~\cite{kkmrel} to the setting of real-valued
functions (see Section \ref{sec:intro} for motivation).  Suppose the data is
distributed according to some distribution $\D$ over $\X \times [0, 1]$. For
$\Y^\prime \supseteq [0, 1]$, let $h : \X \rightarrow \Y^\prime$ be some
function and let $\ell : \Y^\prime \times [0, 1] \rightarrow \reals^+$ be a
loss function. We define the following two \emph{losses} for $f$ with respect
to the distribution $\D$:
\begin{align}
	\falsezero(h; \D) &= \Pr_{(\vx, y) \sim \D}[h(\vx) \neq 0 \wedge y = 0] \label{eqn:false0} 	\\
	\ellgtzero(h; \D) &= \E_{(\vx, y) \sim \D}[ \ell(h(\vx), y) \cdot \indicator(y > 0)], \label{eqn:ellgt0}
\end{align}
where $\indicator(y > 0)$ is $1$ if $y > 0$ and $0$ otherwise. In words,
$\falsezero$ considers the zero-one error on points where the target $y$ equals
$0$ and $\ellgtzero$ considers the loss (or risk) when $y > 0$. Both of these
losses are defined with respect to the distribution $\D$, without conditioning
on the events $y = 0$ or $y > 0$. This is necessary to make efficient learning possible---if the probability of the
events $y = 0$ or $y > 0$ is too small, it is impossible for learning
algorithms to make any meaningful predictions conditioned on those events.

\begin{definition}[Reliable Agnostic Learning]
	We say that a concept class $\C \subseteq [0, 1]^{\X}$ is reliably
	agnostically learnable (reliably learnable for short) with respect to loss
	function $\ell : \Y^\prime \times [0, 1] \rightarrow \reals^+$ (where $[0,
	1] \subseteq \Y^\prime$), if the following holds. For every $\delta,
	\epsilon > 0$, there exists a learning algorithm $\mathcal{A}$ such that,
	for every distribution $\D$ over $\X \times [0, 1]$, when $\mathcal{A}$ is
	given access to examples drawn from $\D$, $\mathcal{A}$ outputs a hypothesis
	$h : \X \rightarrow \Y^\prime$, such that with probability at least $1 -
	\delta$, the following hold:
	\begin{align*}
		\text{(i) } \falsezero(h; \D) &\leq \epsilon,
		&\text{(ii) } \ellgtzero(h; \D) &\leq \min_{c \in \C^+(\D)} \ellgtzero(c) + \epsilon,
	\end{align*}
	where $\C^+(\D) = \{ c \in \C ~|~ \falsezero(c; \D) = 0 \}$. Furthermore, if
	$\X \subseteq \reals^n$ and $s$ is a parameter that captures the
	representation complexity of concepts $c$ in $\C$, we say that $\C$ is
	efficiently reliably agnostically learnable to error $\epsilon$ if $\mathcal{A}$ 
	can output an $h$ satisfying the above conditions with running time that is
	polynomial in $n$, $s$, and $1/\delta$.\footref{ftn:efflearn}
\end{definition}

\subsubsection{Loss Functions}

We have defined agnostic and reliable learning in terms of general loss
functions. Below we describe certain properties of loss functions that are
required in order for our results to hold. Let $\Y$ denote the range of concepts
from the concept class; this will typically be $[-1, 1]$ or $[0, 1]$. Let
$\Y^\prime \supseteq \Y$. We consider loss functions of the form, $\ell :
\Y^\prime \times \Y \rightarrow \reals^+$ and define the following properties:
\label{sec:props}
\begin{itemize}
	\item We say that $\ell$ is \emph{convex in its first argument} if for every
		$y \in \Y$ the function $\ell( \cdot, y)$ is convex.
	\item We say that $\ell$ is \emph{monotone} if for every $y \in \Y$, if
		$y^{\prime\prime} \leq y^{\prime} \leq y$, then $\ell(y^\prime, y) \leq
		\ell(y^{\prime\prime}, y)$ and if $y \leq y^\prime \leq
		y^{\prime\prime}$, $\ell(y^\prime, y) \leq \ell(y^{\prime\prime}, y)$.
		Note that this is weaker than requiring that $|y^\prime - y| \leq |y^{\prime
		\prime} -y|$ implies $\ell(y^\prime, y) \leq \ell(y^{\prime \prime}, y)$.
		This latter condition is not satisfied by several commonly used loss
		functions, \eg hinge loss.
	\item We say that $\ell$ is \emph{$b$-bounded} on the interval $[u, v]$, if for
		every $y \in \Y$, $\ell(y^\prime, y) \leq b$ for $y^\prime \in [u,
		v]$.
	\item We say that $\ell$ is \emph{$L$-Lipschitz} in interval $[u, v]$, if
		for every $y \in \Y$, $\ell(\cdot, y)$ is $L$-Lipschitz in the interval
		$[u, v]$. 
\end{itemize}

The results presented in this work hold for loss functions that are convex,
monotone, bounded and Lipschitz continuous in some suitable interval.
(Monotonicity is not strictly a requirement for our results, but the sample
complexity bounds may be worse for non-monotone loss functions; we point this out when relevant.) These
restrictions are quite mild, and virtually every loss function commonly
considered in (convex approaches to) machine learning satisfy these conditions.
For instance, when $\Y = \Y' = [0,1]$, it is easy to see that any $\ell_p$ loss
function is convex, monotone, bounded by 1 and $p$-Lipschitz for $p \geq 1$. 

\subsection{Kernel Methods} \label{ssec:rkhs}

We make use of kernel methods in our learning algorithms. For completeness, we
define kernels and a few important results concerning kernel methods. The
reader may refer to Hofmann~\etal~\cite{hofmann2008kernel} (or any standard
text) for further details.

Any function $K: \X \times \X \rightarrow \mathbb{R}$ is called a kernel
\cite{mercer}. A kernel $K$ is symmetric if $K(\vx, \vx^\prime) =
K(\vx^\prime,\vx), \forall \vx, \vx^\prime \in \X$; $K$ is positive definite if
$\forall n \in \mathbb{N}, \forall \vx_1, \ldots, \vx_n \in \X$, the $n \times
n$ matrix $\vK$, where $\vK_{i, j} = K(\vx_i, \vx_j)$, is positive
semi-definite.  For any positive definite kernel,  there exists a Hilbert space
$\calH$ equipped with an inner product $\langle \cdot, \cdot \rangle$ and a
function $\psi: \X \rightarrow \calH$ such that $\forall \vx, \vx^\prime \in
\X, K(\vx, \vx^\prime) = \langle \psi(\vx), \psi(\vx^\prime) \rangle$. We refer to
$\psi$ as the \emph{feature map} for $K$.

By convention, we will use $\cdot$ to denote the standard inner product in
$\reals^n$ and $\langle \cdot, \cdot \rangle$ for the inner product in a 
Hilbert Space $\calH$. When $\calH=\reals^n$ for some finite $n$,
we will use  $\langle \cdot, \cdot \rangle$  and $\cdot$ interchangeably. 

We will use the following variant of the polynomial kernel:\vkmnote{I think the
original numbers were slightly off $\psi_d$ is not mapping to $\reals^{n^d}$,
but actually, $\reals^{1 + n + \cdots + n^d}$. Someone please check my
calculuations.}

\begin{definition}[Multinomial Kernel] \label{kernel1}
	Define $\psi_d \colon \R^n \to \R^{N_d}$, where $N_d = 1 + n + \cdots +
	n^d$, indexed by tuples $(k_1, \ldots, k_j) \in [n]^j$ for each $j \in \{0,
	1, \ldots, d \}$, where the entry of $\psi_d(\vx)$ corresponding to tuple
	$(k_1, \ldots, k_j)$ equals $x_{k_1} \cdots x_{k_j}$.  (When $j = 0$ we have
	an empty tuple and the corresponding entry is $1$.) Define kernel ${\MK_d}$
	via: 
	\[
		{\MK_d}(\vx, \vx^\prime) = \langle \psi_d(\vx), \psi_d(\vx^\prime) \rangle = \sum_{j=0}^d (\vx \cdot \vx^\prime)^j.
	\]
Also define $\calH_{{\MK_d}}$ to be the corresponding Reproducing Kernel Hilbert Space (RKHS).
\end{definition}

Observe that ${\MK_d}$ is the sum of standard polynomial kernels (cf.
\cite{wikipedia}) of degree $i$ for $i \in [d]$. However, the feature map
conventionally used for a standard polynomial kernel has only ${n+d \choose d}$
entries and, under that definition involves coefficients of size as large as
$d^{\Theta(d)}$.  The feature map $\psi_d$ used by ${\MK_d}$ avoids these
coefficients by using $N_d$ entries as defined above 
(that is, entries of $\psi_d(\vx)$ are indexed by \emph{ordered} subsets of $[n]$,
while entries of the standard feature map are indexed by \emph{unordered} subsets of $[n]$.)

Let $q \colon \R^n \to \R$ be a multivariate polynomial of total degree $d$. 
We say that a vector $\vv \in \calH_{{\MK_d}}$ \emph{represents} $q$
if $q(\vx) = \langle  \vv, \psi_d(\vx) \rangle$ for all $\vx \in \S^{n-1}$. 
Note that although the feature map $\psi_d$ is fixed, a polynomial
$q$ will have many representations $\vv$ as a vector in
$\calH_{{\MK_d}}$. Furthermore, observe that
the Euclidean norm, $\langle \vv, \vv\rangle$, of these representations may not be equal.

The following example will
play an important role in our algorithms for learning ReLUs. Let $\vw \in \reals^n$ and let $p(t)$ be a univariate degree-$d$ equal to
$\sum_{i = 0}^d \beta_i t^i$ be given. Define the multivariate polynomial
$p_{\vw}(\vx) := p( \vw \cdot \vx)$. 

Consider the representation of $p_{\vw}$ as an element of $\calH_{\MK_d}$
defined as follows: the entry of index $(k_1, \ldots, k_j) \in [n]^j$ of the representation equals
$\beta_j \cdot \prod_{i = 1}^j w_{k_i}$ for $j \in [d]$. Abusing notation,
we use $p_{\vw}$ to denote both the multivariate polynomial and the vector 
in $\calH_{\MK_d}$. 
The following lemma establishes that $p_{\vw} \in \calH_{\MK_d}$ is indeed
a representation of the polynomial $p_{\vw}$, and gives a bound on 
$\langle p_{\vw}, p_{\vw} \rangle$.
The proof follows an analysis applied by Shalev-Shwartz et al. \cite[Lemma 2.4]{SSSS} to
a different kernel (cf. Remark \ref{s4remark} below).
\begin{lemma}\label{coeffbound}
	Let $p(t) = \sum_{i=0}^d \beta_i t^i$ be a given univariate polynomial with
	$\sum_{i=1}^d \beta_i^2 \leq B$. For $\vw$ such that $\lnorm{\vw} \leq 1$,
	consider the polynomial $p_{\vw}(\vx) := p(\vw \cdot \vx)$. Then $p_{\vw}$ is represented
	by the vector $p_\vw \in \calH_{{\MK_d}}$ defined above.
	Moreover,
	$\langle p_{\vw}, p_{\vw} \rangle \leq B$.
\end{lemma}
\begin{proof}
	To see that $p_{\vw}(\vx)= \langle p_{\vw}, \psi_d(\textbf{x})\rangle$ for all $\vx \in \R^n$, observe that 
	\begin{align*}
		p_{\vw}(\vx) &= p(\vw \cdot \vx ) = \sum_{i=0}^d \beta_i \cdot \left(\vw \cdot
		\vx\right)^i \\
		&= \sum_{i=0}^d \sum_{(k_1, \ldots, k_i) \in [n]^i} \beta_i \cdot w_{k_1} \cdot \cdots \cdot w_{k_i} \cdot x_{k_1} \cdot \cdots \cdot x_{k_i} \\ 
		& = \langle p_{\vw} , \psi_d(\vx) \rangle.
	\end{align*}
	Furthermore, we can compute 
	\begin{align*}
		\langle p_{\vw}, p_{\vw} \rangle &= \sum_{i=0}^d \sum_{(k_1, \ldots, k_i)
		\in [n]^i} \beta_i^2 \cdot w_{k_1}^2 \cdot \cdots \cdot w_{k_i}^2 \\
		&= \sum_{i=0}^d \beta_i^2 \cdot \sum_{k_1 \in [n]} w_{k_1}^2 \cdot
		\cdots \cdot \sum_{k_i \in [n]} w_{k_i}^2 \\
		&= \sum_{i=0}^d \beta_i^2 \lnorm{\vw}_2^{2i} = \sum_{i=0}^d
		\beta_i^2 \leq B.
		\end{align*}
\end{proof}

\begin{remark} \label{s4remark}
  	Shalev-Shwartz et al.\cite{SSSS} proved a bound on the Euclidean norm of
	representations of polynomials of the form $p(\vw \cdot \vx)$ in the RKHS
	corresponding to the kernel function $K(\vx,\mathbf{y}) = \frac{1}{1 -
	\frac{1}{2} \langle \vx,\mathbf{y} \rangle}$.  This allowed them to represent functions computed by
	power series, as opposed to polynomials of (finite) degree $d$.
	However, for degree $d$ polynomials, the use of their kernel results in a 
	Euclidean norm bound that is a
	factor of $2^d$ worse than what we obtain from Lemma \ref{coeffbound}.
	This difference is central to our results on noisy polynomial reconstruction
	in Section \ref{sec:agnosticpoly}, where we address this issue in more
	technical detail.
\end{remark}

\label{sec:cskprelims}

\subsection{Generalization Bounds}
\label{sec:generror}
We make use of the following standard generalization bound for hypothesis
classes with small Rademacher complexity. Readers unfamiliar with Rademacher
complexity may refer to the paper of Bartlett and Mendelson~\cite{BM:2002}.

\begin{theorem}[Bartlett and Mendelson \cite{BM:2002}] \label{generalizationbound}
	Let $\D$ be a distribution over $\X \times \Y$ and let $\ell : \Y^\prime
	\times \Y$ (where $\Y \subseteq \Y^\prime \subseteq \reals$) be a
	$b$-bounded loss function that is $L$-Lispschitz in its first argument.  Let
	$\F \subseteq (\Y^\prime)^\X$ and for any $f \in \F$, let $\risk(f; \D) := \E_{(x, y)
	\sim \D}[\ell(f(x), y)]$ and $\hatrisk(f; S) := \frac{1}{m} \sum_{i = 1}^m
	\ell(f(\vx_i), y_i)$, where $S = ((\vx_1, y_1), \ldots,  (\vx_m, y_m))  \sim
	\D^m$. Then for any $\delta > 0$, with probability at least $1 - \delta$
	(over the random sample draw for $S$), simultaneously for all $f \in
	\mathcal{F}$, the following is true:
	\[
		|\risk(f; \D) - \hat{\mathcal{L}}(f; S)| \leq 4 \cdot L \cdot \mathcal{R}_m(\mathcal{F})
		+ 2\cdot b \cdot \sqrt{\frac{\log (1/\delta)}{2m}}
	\]
	where $\mathcal{R}_m(\mathcal{F})$ is the Rademacher complexity of the function
	class $\mathcal{F}$. 
\end{theorem}

We will combine the following two theorems with Theorem
\ref{generalizationbound} above to bound the generalization error of our
algorithms for agnostic and reliable learning.

\begin{theorem}[Kakade et al. \cite{KST:2008}] \label{rademachercomplexity}
	Let $\X$ be a subset of a Hilbert space equipped with inner product $\langle
	\cdot, \cdot \rangle$ such that for each $\vx \in \X$, $\langle \vx, \vx
	\rangle \leq X^2$, and let $\W = \{ \vx \mapsto \langle \vx , \vw \rangle
	~|~ \langle \vw, \vw \rangle \leq W^2 \}$ be a class of linear functions.
	Then it holds that
	\[
		\calR_m(\W) \leq X \cdot W \cdot \sqrt{\frac{1}{m}}.
	\]
\end{theorem}

The following result as stated appears in~\cite{BM:2002} but is originally
attributed to \cite{LT:1991}.

\begin{theorem}[Bartlett and Mendelson \cite{BM:2002}, Ledoux and Talagrand \cite{LT:1991}]
\label{rademachercomplexity2}
	Let $\psi : \reals \rightarrow \reals$ be Lipschitz with constant $L_{\psi}$
	and suppose that $\psi(0) = 0$. Let $\Y \subseteq \mathbb{R}$, and for a function $f \in \Y^{\X}$, let $\psi \circ f$ denote the standard composition of $\psi$ and $f$.
	Finally, for $\F \subseteq \Y^{\X}$, let $\psi \circ \F = \{\psi \circ f \colon f \in \F\}$.
	It holds that $\calR_m(\psi
	\circ \F) \leq 2 \cdot L_{\psi} \cdot \calR_m(\F)$.
\end{theorem}

\subsection{Approximation Theory}

First, we show that the rectifier activation function $\rec(x) = \mathsf{max}(0, x)$ can
be $\eps$-approximated using a polynomial of degree
$O(1/\epsilon)$. This result follows using Jackson's theorem (see, \eg
\cite{New:1964}). For convenience in later proofs, we will require that in fact
the polynomial also takes values in the range $[0, 1]$ on the interval $[-1,
1]$. Of course, this is achieved easily starting from the polynomial obtained
from Jackson's theorem and applying elementary transformations.
\begin{lemma} \label{reluapprox}
	Let $\rec(x) = \mathsf{max}(0, x)$ and $\epsilon \in (0,1)$. There exists a
	polynomial $p$ of degree $O(1/\epsilon)$ such that for all $x \in [-1,1]$,
	$|\rec(x) - p(x)| \leq \epsilon$ and $p([-1, 1]) \subseteq [0, 1]$.
\end{lemma}
\begin{proof}
	We can express $\rec(x) = \mathsf{max}(0, x)$ as $\rec(x) = (x + |x|)/2$.  We know
	from Jackson's Theorem~\cite{New:1964} that there exists a polynomial
	$\tilde{p}$ of degree $O(1/\epsilon)$ such that for all $x \in [-1,1]$,
	$\left| |x| - \tilde{p}(x)\right| \leq \frac{\epsilon}{2 - \epsilon}$. 
	Consider the polynomial $\bar{p}(x) = \frac{\tilde{p}(x) + x}{2}$, which
	satisfies for any $x \in [-1, 1]$, 
	\begin{align*}
		|\rec(x) - \bar{p}(x)| &= \left|\frac{|x| + x}{2} - \frac{\tilde{p}(x) + x}{2}\right| = \left|\frac{|x| - \tilde{p}(x)}{2}\right| \leq \frac{\epsilon}{2 (2 - \epsilon)}.
	\end{align*}

	Finally, let $p(x) = \frac{2 - \epsilon}{2} (\bar{p}(x) - \frac{1}{2}) + \frac{1}{2}$. We have for $x \in [-1, 1]$,
	\begin{align*}
		|\rec(x) - p(x)| &= \frac{\epsilon}{2} |\rec(x)| + \frac{2 - \epsilon}{2}
		\left|\rec(x) - \bar{p}(x)\right| + \frac{1}{2} \left|\frac{2 - \epsilon}{2} -1\right| \leq \epsilon.
	\end{align*}
	Furthermore, it is clearly the case that $p([-1, 1]) \subseteq [0, 1]$.
\end{proof}

We remark that a consequence of the linear relationship between $\rec(x)$ and
$|x|$ is that the degree given by Jackson's theorem is essentially the lowest
possible~\cite{New:1964}.  Lemma \ref{reluapprox} asserts the existence of a
(relatively) low-degree approximation $p$ to the rectifier activation function
$\rec$. We will also require a bound on the sum of the squares of the
coefficients of $p$. Even though Lemma \ref{reluapprox} is non-constructive, we
are nonetheless able to obtain such a bound below via standard interpolation
methods.

\begin{lemma} \label{coeffboundReLU}
	Let $p(t) = \sum_{i=0}^d \beta_i t^i$ be a univariate polynomial of degree
	$d$. Let $M$ be such that $\displaystyle\max_{t \in [-1, 1]} |p(t)| \leq M$.
	Then $\displaystyle\sum_{i=0}^d \beta_i^2 \leq (d + 1) \cdot (4e)^{2d} \cdot
	M^2$. 
\end{lemma}
\begin{proof}
	Lemma 4.1 from Sherstov\cite{She:2012} states that for any polynomial satisfying the
	conditions in the statement of the lemma, the following holds for all
	$i \in \{0, \ldots, d \}$:
	\[ 
		|\beta_i| \leq (4e)^d \max_{j = 0, \ldots, d} \left| p \left( \frac{j}{d} \right) \right|.
	\]
	We then have that 
	\[ 
		\sum_{i = 0}^d \beta_i^2 =  \sum_{i = 0}^d |\beta_i|^2 \leq (d + 1) 
		\cdot (4e)^{2d} \cdot M^2.
	\]
\end{proof}

\begin{theorem} \label{polyapprox}
	Let $\mathcal{C} = \ReLU(n, W)$ (for $W \geq 1$) and $\epsilon \in (0,1)$.
	Let $\X = \S^{n-1}$. For $\vx, \vx^\prime \in \X$, consider the kernel
	${\MK_d}$, with $\calH_{\MK_d}$ and $\psi_d$ the corresponding RKHS and
	feature map (cf. Definition \ref{kernel1}). Then for every $\vw \in
	\reals^n$ with $\lnorm{\vw} \leq W$, there exists a multivariate polynomial
	$p_{\vw}$ of degree at most $O(W/\epsilon)$, such that, for every $\vx \in
	\S^{n-1}$, $|\relu_\vw(\vx) - p_{\vw}(\vx)| \leq \epsilon$. Furthermore,
	$p_\vw(\S^{n-1}) \subseteq [0, W]$ and $p_\vw$ when viewed as a member of
	$\calH_{\MK_d}$ as described in Section~\ref{ssec:rkhs}, satisfies $\langle
	p_{\vw}, p_\vw \rangle \leq W^2 \cdot 2^{O(W/\epsilon)}$.
\end{theorem}
\begin{proof}
	Let $p$ be the univariate polynomial of degree $d = O(W/\epsilon)$ given by
	Lemma~\ref{reluapprox} that satisfies $|p(x) - \rec(x)| \leq
	\frac{\epsilon}{W}$ for $x \in [-1, 1]$.  Let $p(x) = \sum_{i=0}^d \beta_i \cdot
	x^i$; then by Lemma~\ref{coeffboundReLU}, we have $\sum_{i=0}^d \beta_i^2
	\leq (d + 1) \cdot (4e)^{2d} = 2^{O(W/\epsilon)}$ (as $|p(x)| \leq 1$ for $x \in
	[-1, 1]$).

	Let $q$ be the univariate polynomial defined as $q(x) = W \cdot p(x/W)$ for
	$W > 1$.  The degree of $q$ is $d$, the same as that of $p$, and if
	$\alpha_i$ are the coefficients of $q$, we have $\sum_{i=0}^d \alpha_i^2
	\leq W^2 \cdot \sum_{i=0}^d \beta_i^2 \leq W^2 \cdot 2^{O(W/\epsilon)} =
	2^{O(W/\epsilon)}$ (since $W > 1$). Let $p_{\vw}(\vx) = q(\vw \cdot \vx)$.
	Note that $|p_{\vw}(\vx) -\relu_\vw(x)|  = |W \cdot p(\vw \cdot \vx/W) - W
	\cdot \relu_{(\vw/W)}(\vx)| \leq \epsilon$ and $p_\vw(\S^{n-1}) \subseteq
	q([-1,1]) \subseteq [0, W]$. Finally, by applying Lemma~\ref{coeffbound}, we
	get that $\langle p_\vw, p_\vw \rangle \leq W^2 \cdot 2^{O(W/\epsilon)}$.
\end{proof}

\section{Reliably Learning the ReLU}
\label{sec:relu}
In this section, we focus on the problem of reliably learning a single
rectified linear unit with weight vectors of norm bounded by $1$, \ie the
concept class $\ReLU(n, 1)$.  Specifically, our goal is to prove Theorem
\ref{thm:relu} from Section \ref{sec:relutheoremsec}. Below we describe the
algorithm and then give a full proof of Theorem~\ref{thm:relu}.

\subsection{Overview of the Algorithm and Its Analysis} 

In order to reliably learn ReLUs, it would suffice to solve Optimization
Problem~\ref{alg:optprob2} (see Section \ref{sec:intro}). This
mathematical program, however, is not convex; hence, we consider a suitable convex
relaxation.

The convex relaxation optimizes over polynomials of a suitable degree.
Theorem~\ref{polyapprox} shows that any concept in $\ReLU(n, 1)$ can be
uniformly approximated to error $\epsilon$ by a degree $O(1/\epsilon)$
polynomial. It will be more convenient to view this polynomial as an element of
the RKHS $\calH_{\MK_d}$ defined in Definition \ref{kernel1}. Recall that the
corresponding kernel is ${\MK_d}(\vx, \vx^\prime) = \sum_{i=0}^d (\vx \cdot
\vx^\prime)^i$ and the feature map is denoted $\psi_d$. Thus, instead of
minimizing over $\vw$ directly as in Optimization Problem~\ref{alg:optprob2},
Optimization Problem~\ref{alg:optprob3} (below) minimizes over $\vv \in
\calH_{\MK_d}$ of suitably bounded norm. In particular, we know that for any
$\vw$, the corresponding polynomial $p_{\vw}$ that $\eps$-approximates $\max(0,
\vw \cdot \vx)$, when viewed as an element of $\calH_{\MK_d}$, satisfies
$\langle p_{\vw}, p_{\vw} \rangle \leq B = 2^{O(1/\epsilon)}$ (see
Theorem~\ref{polyapprox}).  Recall that $\langle p_{\vw}, \psi_d(\vx) \rangle =
p_\vw(\vx)$. Thus, we have the following optimization problem:

\label{sec:randomlabeljt} 

\begin{algorithm}[H]
	\caption{\label{alg:optprob3}}
	\begin{align*}
		\underset{\vv \in \calH_{\MK_d}}{\text{minimize}} \quad\quad &\sum_{i:~y_i >
		0} \ell(\langle \vv, \psi_d(\vx_i) \rangle, y_i) \\
		\text{subject to} \quad\quad &\langle \vv, \psi_d(\vx_i) \rangle \leq
		\epsilon \quad \text{for all $i$ such that } y_i = 0 \\ 
		&\quad~~\langle \vv, \vv \rangle \leq B
	\end{align*}
\end{algorithm}

Clearly, if $\vw$ is a feasible solution to Optimization
Problem~\ref{alg:optprob2}, then the corresponding element $p_{\vw} \in
\calH_{\MK_d}$ is a feasible solution to Optimization Problem~\ref{alg:optprob3}.
We consider the value of the program for the feasible solution $p_\vw$. For
every $\vx \in \S^{n-1}$, $p_\vw(\vx) = \langle p_\vw, \psi_d(\vx) \rangle \in
[0, 1]$.  Assuming that the loss function $\ell$ is $L$-Lipschitz in its first
argument in the interval $[0, 1]$, we have
\begin{align*}
	\left| \sum_{i :~y_i > 0} \ell(\relu_\vw(\vx), y_i) - \sum_{i :~y_i > 0}
	\ell (\langle p_{\vw}, \psi_d(\vx) \rangle, y_i) \right| &\leq |\{ i ~|~ y_i >
	0 \}| \cdot L \cdot \epsilon.
\end{align*}
Thus, an optimal solution to Optimization Problem~\ref{alg:optprob3} achieves a
loss on the training data that is within $|\{ i ~|~ y_i > 0 \}| \cdot L \cdot
\epsilon$ of that achieved by the optimal solution to Optimization
Problem~\ref{alg:optprob2}. 

While Optimization Problem~\ref{alg:optprob3} is convex, it is still not
trivial to solve efficiently. For one, the RKHS $\calH_{\MK_d}$ has dimension $n^{\Theta(d)}$. 
However, materializing such vectors explicitly requires
$n^{\Theta(d)}$ time, and Theorem \ref{thm:relu}  promises a learning algorithm
with runtime $2^{O(1/\eps)} \cdot n^{O(1)} \ll n^{O(d)}$.  As in Shalev-Shwartz
et al. \cite{SSSS}, we apply the Representer Theorem (see \eg
\cite{CS-T:2000}), to guarantee that Optimization Problem~\ref{alg:optprob3}
can be solved in time that is polynomial in the number of samples used.  

The Representer Theorem states that for any vector $\vv$, there exists
a vector $\vv_{\alpha} = \sum_{i = 1}^m \alpha_i \psi_d(\vx_i)$ for
$\alpha_1, \dots, \alpha_m \in \mathbb{R}$  such that
the loss function of Optimization Problem~\ref{alg:optprob3} subject
to the constraint $\langle \vv, \vv \rangle \leq B$ does not increase
when $\vv$ is replaced with $\vv_{\alpha}$.  
Crucially, we may further constrain these vectors $\vv_{\alpha}$ to obey the
inequality $\langle \vv_{\alpha}, \psi_d(\vx_i) \rangle \leq \epsilon$ for all $i$ such that $y_i=0$.
Thus, Optimization Problem~\ref{alg:optprob3} can be reformulated in
terms of the variable vector $\valpha = (\alpha_1, \ldots,
\alpha_m)$. This mathematical program is described as Optimization
Problem~\ref{alg:final} below.

\begin{algorithm}[H]
	\caption{\label{alg:final}}
	\begin{align*}
		\underset{\valpha \in \mathbb{R}^m}{\text{minimize}} \quad\quad &\sum_{i :
		y_i > 0} \ell\left(\sum_{j=1}^m \alpha_j {\MK_d}(\vx_j,\vx_i), y_i \right) \\
		\text{subject to} \quad\quad &\quad~\sum_{j=1}^m \alpha_j \cdot
		{\MK_d}(\textbf{x}_j,\textbf{x}_i)\leq \epsilon \quad \text{for all $i$ such
		that } y_i = 0 \\ 
		&\sum_{i,j=1}^m \alpha_i \cdot \alpha_j \cdot {\MK_d}(\textbf{x}_i,\textbf{x}_j) \leq
		B
	\end{align*}
\end{algorithm}

Let $\vK$ denote the $m \times m$ \emph{Gram} matrix whose $(i, j)\th$ entry is
${\MK_d}(\vx_i, \vx_j)$. Using the notation $\valpha = (\alpha_1, \ldots, \alpha_m)$,
the last constraint is equivalent to $\valpha^T \vK \valpha \leq B$. As ${\MK_d}
\succeq 0$, this defines a convex subset of $\reals^m$. The remaining
constraints are linear in $\valpha$ and whenever the loss function $\ell$ is
convex in its first argument, the resulting program is convex. Thus,
Optimization Problem~\ref{alg:final} can be solved in time polynomial in $m$. 

\subsection{Description of the Output Hypothesis} 

\label{ssec:hypothesis}

Let $\valpha^*$ denote an optimal solution to Optimization
Problem~\ref{alg:final} and let $f(\cdot) = \sum_{i=1}^m \alpha^*_i
{\MK_d}(\vx_i, \cdot)$. To obtain strong
bounds on the generalization error of our hypothesis,
our algorithm does not simply output $f$ itself.
The obstacle is that, although $f$ (viewed as an
element of $\calH_{\MK_d}$) satisfies $\langle f, f\rangle \leq B$, the best
bound we can obtain on $|f(\vx)| = |\langle f, \vx \rangle|$ for $\vx \in
\S^{n-1}$ is $\sqrt{B}$ by the Cauchy-Schwartz inequality. Observe that for
many commonly used loss functions, such as the squared loss, this may result in
a very poor Lipschitz constant and bound on the loss function, when applied to
$f$ in the interval $[-\sqrt{B}, \sqrt{B}]$ (recall that the only bound we have
is $B = 2^{O(1/\epsilon)}$). Hence, a direct application of standard
generalization bounds (cf. Section \ref{sec:generror}) yields a very weak bound
on the generalization error of $f$ itself.  For example, suppose $y \in \{0,
1\}$ and consider the loss function $\ell(y^\prime, y) = \exp\left(-y^\prime(2y
- 1) + 1\right) - 1$ if $y^\prime (2 y - 1) \leq 1$ and $\ell(y^\prime, y) = 0$
otherwise (this loss function is like the hinge loss, but the linear side is
replaced by an exponential). The Lipschitz constant of $\ell$ on the interval
$[-\sqrt{B}, \sqrt{B}]$ is exponentially large in $B$, which would lead to a
sample complexity bound that is \emph{doubly}-exponentially large in
$1/\epsilon$. 

To address this issue, we will ``clip'' the function to always output a value
between $[0, 1]$: 
\begin{definition} 
        Define the function $\clip_{a, b} : \reals \rightarrow [a, b]$ as
	follows: $\clip_{a, b}(x) = a$ for $x \leq a$, $\clip_{a, b}(x) = x$ for $a
	\leq x \leq b$ and $\clip_{a, b}(x) = b$ for $b \leq x$.
\end{definition}
The hypothesis $h$ output by our algorithm is as follows.
\[
	h(x) = \begin{cases}
		0 & \text{if } \clip_{0, 1}(f(x)) \leq 2 \cdot \epsilon \\
		\clip_{0, 1}(f(x)) & \text{otherwise}.
			\end{cases}
\]

We use a fact due to Ledoux and Talagrand on the Rademacher complexity of
composed function classes (Theorem \ref{rademachercomplexity2}) to bound the
generalization error. Clipping comes at a small cost, in the sense that it forces us to require 
that the loss
function be monotone. However, we can handle non-monotone losses if the
output hypothesis is not clipped, albeit with sample complexity bounds that
depend polynomially on the Lipschitz-constant and bound of the loss in the
interval $[-\sqrt{B}, \sqrt{B}]$ as opposed to $[0, 1]$. 

\ignore{
{\color{red}{Note that, regardless of whether we clip
or not, the sample complexity and running time of our algorithm has an exponential
dependence on the Lipschitz constant $L$ of the loss function in the interval $[0, 1]$, for reasons
unrelated to generalization error (cf.
Theorem~\ref{thm:formal} and the discussion at the end of Section \ref{sec:33}). 
However, $L$ is a constant for virtually
every imaginable loss function.}} }

\subsection{Formal Version of Theorem~\ref{thm:relu} and Its Proof}
\label{sec:33}
The rest of this section is devoted to the proof of Theorem~\ref{thm:relu} (or,
more precisely, its formal variant Theorem \ref{thm:formal} below, which makes
explicit the conditions on the loss function $\ell$ that are required for the
theorem to hold). In particular, we show that whenever the sample size $m$ is a
sufficiently large polynomial in $2^{O(1/\epsilon)}$, $n$, and
$\log(1/\delta)$, %
%
%
%
the hypothesis $h$ output by the algorithm satisfies $\falsezero(h; \D) =
O(\epsilon)$ and $\ellgtzero(h; \D) \leq \min_{c \in \C^+(\D)} \ellgtzero(c) +
O(\epsilon)$, where $\C^+(\D) = \{ \relu_\vw \in \ReLU(n, 1) ~|~
\falsezero(\relu_\vw; \D) = 0 \}$.  Rescaling $\epsilon$ appropriately completes
the proof of Theorem~\ref{thm:formal}. 

\begin{theorem}[Formal Version of Theorem \ref{thm:relu}]
	\label{thm:formal} Let $\X  = \S^{n-1}$ and $\Y = [0, 1]$. The concept class
	$\ReLU(n, 1)$ is reliably learnable with respect to any loss function that
	is convex, monotone, and $L$-Lipschitz and $b$-bounded in the interval $[0,
	1]$. The sample complexity and running time of the
	algorithm is polynomial in $n$, $b$, $\log(1/\delta)$ and
	$2^{O(L/\epsilon)}$. In particular, $\ReLU(n, 1)$ is
	learnable in time polynomial in $n$, $b$ and $\log(1/\delta)$ up to accuracy
	$\epsilon \geq \epsilon_0 = \Theta(L/\log(n))$,
	where $L$ is the Lipschitz constant of the loss function in the interval
	$[0, 1]$.
\end{theorem}

\begin{proof}
	In order to prove the theorem, we need to bound the following two losses for
	the output hypothesis $h$.
	\begin{align}
		\falsezero(h; \D) &= \Pr_{(\vx, y) \sim \D} [h(\vx) \neq 0 \wedge y = 0] \label{eqn:fixfzl} \\
		\ellgtzero(h; \D) &= \E_{(\vx, y) \sim \D}[\ell(h(\vx), y) \cdot \indicator(y > 0)] \label{eqn:fixgtzl}
	\end{align}

	First, we analyze $\falsezero(h; \D)$; in order to analyze this loss, it is
	useful to consider a slightly different loss function that is
	$\left(1/\epsilon\right)$-Lipschitz in its first argument,
	$\ellepszo(y^\prime, y)$. We define this loss separately for the case when
	$y > 0$ and $y = 0$. For $y > 0$, we define $\ellepszo(y^\prime, y) := 0$
	for all $y^\prime$. For $y = 0$, we define
	\[ 
		\ellepszo(y^\prime, 0) := 
		\begin{cases}
			0 & \text{if } y^\prime \leq \epsilon \\
			\frac{y^\prime - \epsilon}{\epsilon} & \text{if } \epsilon < y^\prime \leq 2 \cdot \epsilon \\
			1 & \text{if } 2 \cdot \epsilon < y^\prime.
		\end{cases}
	\]
	For $f : \X \rightarrow \Y$, let $\Lepszo(f; \D) := \E_{(\vx, y) \sim
	\D}[\ellepszo(f(\vx), y)]$. Let $d = O(1/\epsilon)$ be such that
	Theorem~\ref{polyapprox} applies for the class $\ReLU(n, 1)$, and $\psi_d$
	and $\calH_{\MK_d}$ the corresponding feature map and Hilbert space.  Define
	$\F_B \subset \calH_{\MK_d}$ as the set of all $f \in \calH_{\MK_d}$ such
	that $\langle f, f\rangle \leq B$. Observe that for all $\vx \in
	\X=\S^{n-1}$, $\langle \psi_d(\vx), \psi_d(\vx) \rangle \leq \sum_{i=0}^d
	(\vx \cdot \vx)^i = d+1$.  Moreover, the function $\clip_{0, 1} : \reals
	\rightarrow [0, 1]$ satisfies, $\clip_{0, 1}(0) = 0$, and $\clip_{0, 1}$ is
	$1$-Lipschitz.  Thus, Theorems~\ref{rademachercomplexity} and
	\ref{rademachercomplexity2} imply the following:
	\begin{align}
		\calR_m(\F_B) &\leq  \sqrt{\frac{(d+1) \cdot B}{m}}, \label{rcbound} \\
		\calR_m(\clip_{0, 1} \circ \F_B) &\leq 2 \cdot \sqrt{\frac{(d+1) \cdot B }{m}}
		\label{rcbound2}
	\end{align} 

	The loss function $\ellepszo$ is $(1/\epsilon)$-Lipschitz in its first
	argument and $1$-bounded on all of $\reals$, so in particular in the
	interval $[0, 1]$; the loss function $\ell$ (used for $\ellgtzero$) is
	$L$-Lipschitz in its first argument and $b$-bounded in the interval $[0, 1]$
	(by assumption in the theorem statement). We assume the following bound on
	$m$ (note that it is polynomial in all the required factors):
	\begin{align}
		m &\geq \frac{1}{\epsilon^2} \left( 8 \max\{L, \epsilon^{-1} \} \sqrt{(d+1) \cdot
		B} + \max\{b, 1\} \cdot \sqrt{2 \log \frac{1}{\delta}} \right)^2.
		\label{eqn:m-bound}
	\end{align} \medskip

	\noindent{\bf Representative Sample Assumption}: In the rest of the proof we
	assume that for the sample $S \sim \D^m$ used in the algorithm, it is the
	case that for loss functions $\ellepszo$ and $\ell$ and for all $f \in
	\F_B$, the following hold:
	\begin{align}
		|\Lepszo(f; \D) - \hatLepszo(f; S)| &\leq \epsilon \label{eqn:lepszo-gen} \\
		|\ellgtzero(\clip_{0, 1} \circ f; \D) - \hatellgtzero(\clip_{0, 1} \circ f; S)| &\leq \epsilon \label{eqn:l-gen} 
	\end{align}
	Theorems~\ref{generalizationbound},~\ref{rademachercomplexity}
	and~\ref{rademachercomplexity2} together with the bounds on the Rademacher
	complexity given by~\eqref{rcbound} and~\eqref{rcbound2} and the facts that
	$\ellepszo$ is $1/\epsilon$-Lipschitz and $1$-bounded on $\reals$ and that
	$\ell$ is $L$-Lipschitz and $b$-bounded on $[0, 1]$, imply that for $m$
	satisfying~\eqref{eqn:m-bound}, this is the case with probability at least
	$1 - 2 \delta$; we allow the algorithm to fail with probability
	$2 \delta$. \medskip

	Now consider the following to bound $\falsezero(h; \D)$.
	\begin{align}
		\falsezero(h; \D) &= \Pr_{(\vx, y) \sim \D}[h(\vx) > 0 \wedge y = 0] \nonumber \\
		&\leq \E_{(\vx, y) \sim \D}[\ellepszo(f(\vx), y)] \label{eqn:replacebyloss} \\
		&= \Lepszo(f; \D) \nonumber \\
		&\leq \hatLepszo(f; S) + \epsilon \leq \epsilon, \label{eqn:applybm} 
	\end{align}
	Above in~(\ref{eqn:replacebyloss}), we use the fact that for any $\vx$ such
	that $h(\vx) > 0$, it must be the case that $f(\vx) > 2 \cdot \epsilon$ and
	hence if $h(\vx) > 0$ and $y = 0$, then $\ellepszo(f(\vx), y) = 1$.
	Inequality (\ref{eqn:applybm}) holds under the representative sample
	assumption using~\eqref{eqn:lepszo-gen} (note that we have already accounted
	for the fact that the algorithm may fail with probability $O(\delta)$). 

	Next we give bounds on $\ellgtzero(h; \D)$. We observe that for a loss
	function $\ell$ that is convex in its first argument, monotone,
	$L$-Lipschitz, and $b$-bounded in the interval $[0, 1]$, the following holds
	for any $y \in (0, 1]$:
	\begin{align}
		\ell(h(\vx), y) &\leq \ell(\clip_{0, 1}(f(\vx)), y) + 2  \epsilon  L
		\label{eqn:bound-loss-of-h-by-f}
	\end{align}
	Clearly, whenever $f(\vx) > 2 \epsilon$ or $f(\vx) < 0$, the above statement is
	trivially true.  If $f(\vx) \in [0, 2 \epsilon]$ the statement follows from
	the $L$-Lipschitz continuity of $\ell(\cdot, y)$ in the interval $[0, 1]$.

	Let $\vw \in \reals^n$ be such that $\falsezero(\relu_\vw; \D) = 0$ and let
	$p_\vw$ be the corresponding polynomial 
	$\eps$-approximation in $\calH_{\MK_d}$ (cf. Theorem~\ref{polyapprox}).  Then
	consider the following:
	\begin{align}
		\ellgtzero(h; \D) &= \E_{(\vx,  y) \sim \D}\left[\ell(h(\vx), y) \cdot \indicator(y > 0) \right] \nonumber \\
		&\leq \E_{(\vx, y)\sim \D}\left[ \ell(\clip_{0, 1}(f(\vx)), y) \cdot \indicator(y > 0)\right] + 2 \epsilon  L \label{eqn:use-h-by-f} \\
		&= \ellgtzero(\clip_{0, 1}(f); \D) + 2  \epsilon  L \nonumber \\
		&\leq \hatellgtzero(\clip_{0, 1}(f); S) + \epsilon + 2 \epsilon  L \label{eqn:apply-Rad} \\
		&\leq \hatellgtzero(f; S) + \epsilon + 2 \epsilon L \label{eqn:relax-clip} \\
		&\leq \hatellgtzero(p_{\vw}; S) + \epsilon + 2  \epsilon  L \label{eqn:optimality} \\
		&= \hatellgtzero(\clip_{0, 1} \circ p_{\vw}; S) + \epsilon + 2  \epsilon  L \label{eqn:addclip} \\
		&\leq \ellgtzero(\clip_{0, 1} \circ p_{\vw}; \D) + 2  \epsilon + 2  \epsilon  L \label{eqn:apply-Rad2} \\
		&\leq \ellgtzero(p_{\vw}; \D) + 2  \epsilon + 2  \epsilon  L \label{eqn:removeclip} \\
		&\leq \ellgtzero(\relu_{\vw}; \D) + 2  \epsilon + 3   \epsilon L \label{eqn:use-lipschitz}
	\end{align}

Step~\eqref{eqn:use-h-by-f} is obtained simply by
applying~\eqref{eqn:bound-loss-of-h-by-f}. Step~\eqref{eqn:apply-Rad} follows
using the representative sample assumption using~\eqref{eqn:l-gen}.
Step~\eqref{eqn:relax-clip} follows by the monotone property of $\ell(\cdot,
y)$; in particular, it must always be the case that either $y \leq \clip_{0,
1}(f(\vx)) \leq f(\vx)$ or $f(\vx) \leq \clip_{0,1}(f(\vx)) \leq y$; thus
$\ell(\clip_{0, 1}(f(\vx)), y) \leq \ell(f(\vx), y)$.
Step~\eqref{eqn:optimality} follows from the fact that $f$ is the optimal
solution to Optimization Problem~\ref{alg:final} and $p_\vw$ is a
\emph{feasible} solution. Steps~\eqref{eqn:addclip} and~\eqref{eqn:removeclip}
use the fact that $\clip_{0, 1} \circ p_\vw = p_\vw$ as $p_\vw(\S^{n-1})
\subseteq [0, 1]$.  Step~\eqref{eqn:apply-Rad2} follows under the
representative sample assumption using~\eqref{eqn:l-gen}.  And finally,
Step~\eqref{eqn:use-lipschitz} follows as both $\relu_{\vw}(\vx) \in [0, 1]$
and $p_{\vw}(\vx) \in [0, 1]$ for $\vx \in \S^{n-1}$, $|p_\vw(\vx) -
\relu_\vw(\vx)| \leq \epsilon$  and the $L$-Lipschitz continuity of $\ell$ in
the interval $[0, 1]$. 

As the argument holds for any $\vw \in \S^{n-1}$ satisfying
$\falsezero(\relu_\vw; \D) = 0$ this completes the proof of theorem by
rescaling $\epsilon$ to $\epsilon/(2 + 3L)$ and $\delta$ to $\delta/2$.
\end{proof}

\subsubsection*{Discussion: Dependence on the Lipschitz Constant}

Theorem~\ref{thm:formal} gives a sample complexity and running time bound that
is polynomial on $2^{O(L/\epsilon)}$ (in addition to being polynomial in other
parameters).  Recall that, here, $L$ is the Lipschitz constant of the loss function $\ell$ on
the interval $[0, 1]$.  For many loss functions, such as
$\ell_p$-loss for constant $p$, hinge loss, logistic loss, \etc, the value of
$L$ is a constant. Nonetheless, it is instructive to examine
why we obtain such a dependence $L$, and identify some restricted settings in which
this dependence can be avoided. 

The dependence of our running time and sample complexity bounds on $L$ 
arises due to Steps~\eqref{eqn:bound-loss-of-h-by-f}
and~\eqref{eqn:use-lipschitz} in the proof of Theorem~\ref{thm:formal}, where
the excess error compared to the optimal ReLU is bounded above by $O(L \epsilon)$. This 
requires us to start with a
polynomial that is an $O(\epsilon/L)$-uniform approximation to the $\rec$
activation function, to ensure excess error at most $\eps$. We showed that such an approximating polynomial exists,
with degree $O(L/\eps)$ and with coefficients whose squares sum to $2^{O(L/\eps)}$.

It is sometimes possible to avoid this exponential dependence on $L$ in the setting of agnostic learning (as opposed to reliable learning).
Indeed, in the case of agnostic learning there is no
need to threshold the output at $2 \epsilon$ (this thresholding contributed $2 \eps L$ to our bound
on the excess error established in {Inequality~(\ref{eqn:bound-loss-of-h-by-f})}); simply clipping the output to be in the
range of $\Y$ suffices.

\subsection{An Implication for Learning Convex Neural Networks} \label{bach}

In a recent work, Bach~\cite{Bac:2014} considered convex relaxations of
optimization problems related to learning neural networks with a single hidden
layer and non-decreasing homogeneous activation function.\footnote{His setting
allows potentially uncountably many hidden units along with a sparsity-inducing
regularizer.}  One specific problem raised in his paper~\cite[Sec.
6]{Bac:2014} is understanding the computational complexity of
the following problem.

\begin{problem}[Incremental Optimization Problem~\cite{Bac:2014}]
	\label{prob:Bach}
	Let $\langle (\vx_i, y_i) \rangle_{i=1}^m \in (\S^{n-1} \times [-1, 1])^m$.
	Find a $\vw \in \S^{n-1}$ that maximizes $\frac{1}{m} \sum_{i = 1}^m y_i \cdot 
	\relu_\vw(\vx_i)$.
\end{problem}

While Bach~\cite{Bac:2014} considers the setting where $y_i \in \reals$, rather
than $[-1, 1]$, we focus on the case when $y_i \in [-1, 1]$. 
The problem as posed above is an optimization
problem on a finite dataset that requires the output solution to be from a
specific class, in this case a ReLU.  In our setting, this can be rephrased as
a (proper) learning problem where the goal is to output a hypothesis that has
expected loss, defined by $\ell(y^\prime, y) = -y^\prime \cdot y$, not much
larger than the best possible ReLU, given access to draws from a distribution
over $\S^{n-1} \times [-1, 1]$.  Here, we relax this goal to improper learning,
where the algorithm is permitted to output a hypothesis that is not itself a
ReLU.  The same approach as used in the
proof of Theorem~\ref{thm:formal} can be used to give a
polynomial-time approximation scheme for approximately solving this problem 
to within $\epsilon$ of optimal, in time $2^{O(1/\epsilon)} \cdot
n^{O(1)}$. 
  
We describe the modified algorithm and the minor differences in the proof below. \medskip 
\begin{algorithm}[H]
	\caption{\label{alg:optprob3a}}
	\begin{align*}
		\underset{\valpha \in \reals^m}{\text{minimize}} \quad\quad &\sum_{i = 1}^m 
		\ell\left(\sum_{j = 1}^m \alpha_j \MK_d (\vx_j, \vx_i), y_i\right) \\ 
		\text{subject to}\quad~~ &\sum_{i, j = 1}^m \alpha_i \alpha_j \MK_d(\vx_i, \vx_j) \leq B
	\end{align*}
\end{algorithm}
The loss function used is $\ell(y^\prime, y) = -y^\prime y$. Let $\valpha^*$ denote an
optimal solution to Optimization Problem~\ref{alg:optprob3a} and let $f(\cdot)
= \sum_{i=1}^m \alpha^*_i \MK_d(\vx_i, \cdot)$. In Problem~\ref{prob:Bach},
there is no reliability required and hence we do not threshold negative (or
sufficiently small positive) values as was done in
Section~\ref{ssec:hypothesis}. Likewise, we do not clip the function $f$; this
is because while the loss function $\ell(y^\prime, y) = -y^\prime y$ is indeed
convex in its first argument, $1$-Lipschitz on $\reals$, and $\sqrt{B}$-bounded
on the interval $[-\sqrt{B}, \sqrt{B}]$ (for $y \in [-1, 1]$; note that
$|f(\vx)| \leq |\langle f, \psi_d(\vx)| \rangle \leq \sqrt{B}$ by the
Cauchy-Schwartz inequality), it is very much \emph{not} monotone.  Thus, it is
no longer the case that $\clip_{-1, 1}(f)$ is a better hypothesis than $f$
itself. We observe that the proof of Theorem~\ref{thm:formal} only makes use of
the monotone nature of $\ell$ to conclude that expected loss of $\clip_{0, 1}
\circ f$ is less than that of $f$. As we no longer output a clipped hypothesis,
this is not necessary.

\begin{theorem}
	\label{thm:bach} Given
	i.i.d. examples $(\vx_i, y_i)$ drawn from an (unknown) distribution $\D$ over $\S^{n-1} \times [-1, 1]$, 
	there is an algorithm that outputs a hypothesis $h$ such that $\E_{(\vx,y)
	\sim \D}[-y \cdot h(\vx)] \leq \min_{\vw \in \S^{n-1}} \E_{(\vx,y)
	\sim \D}[-y \cdot \relu_\vw(\vx)] + \eps$. The algorithm runs in time $2^{O(1/\epsilon)} \cdot
n^{O(1)}$. 
	\end{theorem}

\subsection{Noisy Polynomial Reconstruction over $\S^{n-1}$}
\label{sec:agnosticpoly}
In the noisy polynomial reconstruction problem, a learner is given access to
examples drawn from a distribution and labeled according to a function $f(\vx) =
p(\vx) + w(\vx)$ where $p$ is a polynomial and $w$ is an arbitrary function
(corresponding to noise).  We will consider a more general scenario, where a
learner is given sample access to an {\em arbitrary} distribution $\D$ on
$\S^{n-1} \times [-1,1]$ and must output the best fitting polynomial with
respect to some fixed loss function.  We say that the reconstruction is {\em
proper} if, given a hypothesis $h$ encoding a multivariate polynomial, we can
obtain any coefficient of our choosing in time polynomial in $n$. 

Note that noisy polynomial reconstruction as defined above is
equivalent to the problem of agnostically learning multivariate
polynomials.  We give an algorithm for noisy polynomial reconstruction
whose runtime is $\poly(B, n, d, 1/\eps)$, where $B$ is an upper bound
on the sum of the squared coefficients of the polynomial in the
standard monomial basis.  Throughout this section, we refer to the sum
of the squared coefficients of $p$ as the \emph{weight} of $p$.

Analogous problems over the Boolean domain are thought to be
computationally intractable.  Andoni et al. \cite{AndoniPV014} were the
first to observe that over non-Boolean domains, the problem admits
some non-trivial solutions.  In particular, they gave an algorithm
that runs in time $\poly(B, n, 2^d, 1/\eps)$ with the requirement that the
underlying distribution be product over the unit cube (and that the
noise function is structured).


Consider a multivariate polynomial $p$ of degree $d$ such that sum of the
squared coefficients is bounded by $B$. Denote the coefficient of monomial
$x_1^{i_1} \cdots x_n^{i_n}$ by $\beta(i_1, \dots, i_n)$ for $(i_1, \ldots,
i_n) \in \{0, \ldots, d\}^n$. We have
\begin{align}
	p(\vx) = \sum_{\substack{(i_1, \dots, i_n) \in [d]^n \\ i_i + \dots + i_n
	\leq d}} \beta(i_1, \dots, i_n) x_1^{i_1} \cdots x_n^{i_n} \label{pdef}
	\intertext{such that}
	\notag \sum_{\substack{(i_1, \ldots, i_n) \in \{0, \ldots, d\}^n\\ i_i + \dots +
	i_n \leq d}}  \beta(i_1, \dots, i_n)^2 \leq B.
\end{align}
Let $M$ be the map that takes an ordered tuple $({k_1}, \ldots, {k_j}) \in
[n]^j$ for $j \in [d]$ to the tuple $(i_1, \ldots, i_n) \in \{0, \ldots, d\}^n$
such that $x_{k_1}\cdots x_{k_j} = x_1^{i_1} \cdots x_n^{i_n}$. Let $C(i_1,
\ldots, i_n)$ be the number of distinct orderings of the $i_j$'s for $j \in [n]$; $C(i_1, \dots, i_n)$
which can be computed from the multinomial theorem (cf. \cite{wiki_multi}).
Observe that the number of tuples that $M$ maps to $(i_1, \ldots,
i_n)$ is precisely $C(i_1, \ldots, i_n)$.

Recall that  $\calH_{{\MK_d}}$ denotes the RKHS from
Definition \ref{kernel1}. Observe that the polynomial $p$ from Equation \eqref{pdef} is represented by the vector $\vv_p \in \calH_{{\MK_d}}$ defined as follows. For $j \in [d]$, entry $(k_1, \dots, k_j)$ 
of $\vv_p$ equals
\[ \frac{\beta\left(M\left({k_1}, \ldots , {k_j}\right)\right)}{C\left(M\left(k_1, \ldots, k_j\right)\right)}.\]
It is easy to see that $\vv_p$ as defined represents $p$. Indeed,
\begin{align*}
	\langle\vv_p, \psi_d(\vx) \rangle & = \sum_{j=0}^d \sum_{(k_1, \ldots, k_j) \in
	[n]^j} \frac{\beta\left(M\left({k_1}, \ldots , {k_j}\right)\right)}{C\left(M\left(k_1, \ldots, k_j\right)\right)}
	x_{k_1} \cdots x_{k_j} \\
	& = \sum_{j=0}^d \sum_{\substack{(i_1, \ldots, i_n) \in \{0, \ldots, d\}^n
	\\ i_i + \dots + i_n = j}} C\left(i_1, \dots, i_n \right) \frac{\beta\left(i_1, \dots,
	i_n\right)}{C\left(i_1, \dots, i_n \right)} x_1^{i_1} \cdots x_n^{i_n}  = p(\vx).
\end{align*}
Furthermore, we can compute,
\begin{align*}
	\langle\vv_p, \vv_p\rangle & = \sum_{j=0}^d \sum_{(k_1, \ldots, k_j) \in
	[n]^j} \frac{\beta\left(M\left({k_1}, \ldots , {k_j}\right)\right)^2}{C\left(M\left(k_1, \ldots, k_j\right)\right)^2}
	\\
	& = \sum_{j=0}^d \sum_{\substack{(i_1, \ldots, i_n) \in \{0, \ldots, d\}^n
	\\ i_i + \dots + i_n = j}} C\left(i_1, \dots, i_n \right) \frac{\beta\left(i_1, \dots,
	i_n\right)^2}{C\left(i_1, \dots, i_n \right)^2} \\
	& \leq \sum_{j=0}^d \sum_{\substack{(i_1, \ldots, i_n) \in \{0, \ldots,
	d\}^n \\ i_i + \dots + i_n = j}} \beta\left(i_1, \dots, i_n \right)^2 \leq B.\\
\end{align*}

\medskip \noindent \textbf{Overview of the Algorithm.}  Let ${\cal C}$
be the class of all multivariate polynomials and let $S =
\{(\vx_1,y_1), \ldots, (\vx_m, y_m)\}$ be a training set of examples
drawn i.i.d. from some arbitrary distribution $\D$ on
$\S^{n-1} \times [-1,1]$. Similar to Optimization Problem
\ref{alg:optprob3} in Section \ref{sec:randomlabeljt}, we wish to
solve Optimization Problem \ref{alg:optprobdie} below.

\begin{algorithm}[H]
	\caption{\label{alg:optprobdie}}
	\begin{align*}
		\underset{\vv \in \calH_{\MK_d}}{\text{minimize}} \quad\quad &\sum_{i = 
		1}^m \ell(\langle \vv,  \psi_d(\vx_i)\rangle, y_i) \\
		\text{subject to} 
		&\quad~~ \langle \vv ,  \vv \rangle  \leq B
	\end{align*}
\end{algorithm}

Notice from the previous analysis, a degree $d$ polynomial $p$ can be represented as 
a vector $\vv_p \in \calH_{{\MK_d}}$ such that $p(\vx) = \langle \vv_p, \psi_d(\vx) \rangle$ for all $\vx \in \S^{n-1}$, and $\langle \vv_p, \vv_p \rangle \leq
B$. 
Thus, $\vv_p$ is a
feasible solution to  Optimization Problem \ref{alg:optprobdie}.  Optimization
Problem \ref{alg:optprobdie} can easily be solved in time $\poly(n^d)$, but
this runtime is not polynomial in $B$ and $n$.  Instead, just as in Section
\ref{sec:randomlabeljt}, we use the Representer Theorem to solve Optimization
Problem~\ref{alg:optprobdie} in time that is polynomial in the number of
samples used. Specifically, the Representer Theorem states that there is an
optimal solution to Optimization Problem~\ref{alg:optprobdie} of the form $\vv
= \sum_{i = 1}^m \alpha_i \psi_d(\vx_i)$ for some values $\alpha_1, \dots,
\alpha_m \in \mathbb{R}$. Thus, Optimization Problem~\ref{alg:optprobdie} can
be reformulated in terms of the variable vector $\valpha = (\alpha_1, \ldots,
\alpha_m)$. This mathematical program is described as Optimization
Problem~\ref{alg:finaldie} below.

\begin{algorithm}[H]
	\caption{\label{alg:finaldie}}
	\begin{align*}
		\underset{\valpha \in \mathbb{R}^m}{\text{minimize}} \quad\quad &\sum_{i =
		1}^m \ell\left(\sum_{j=1}^m \alpha_j {\MK_d}(\vx_j,\vx_i), y_i \right) \\
		\text{subject to} 
		&\sum_{i,j=1}^m \alpha_i \cdot \alpha_j \cdot {\MK_d}(\textbf{x}_i,\textbf{x}_j) \leq
		B
	\end{align*}
\end{algorithm}

Via a standard analysis identical to that of Section \ref{sec:randomlabeljt},
Optimization Problem~\ref{alg:finaldie} is a convex program and can be solved
in time polynomial in $m$, $n$, and $d$. 
Let $\valpha^*$ denote an optimal solution to Optimization
Problem~\ref{alg:finaldie} and let $f(\cdot) = \sum_{i=1}^m \alpha_i^* {\MK_d}(\vx_i,
\cdot)$. 
The hypothesis $h$ output by our algorithm is as follows.
\[
	h(\vx) =  \clip_{-1, 1}(f(\vx)).
\]

Observe that $h \in \clip_{-1, 1} \circ \mathcal{C}$. 

\subsection{Proper Learning}

As discussed in Section \ref{ssec:hypothesis}, we require clipping to avoid a
weak bound on the generalization error for general loss functions.  If,
however, we consider learning with respect to any $\ell_p$ loss for constant $p
\geq 1$, it can be shown that we can do without clipping (with only a polynomial
factor 
increase in sample complexity). In this case, the learner $h = f$ is a
\emph{proper learner} in the following sense. Recalling the feature map $\psi_d$ associated with
${\MK_d}$ from Definition \ref{kernel1}, we can compute the coefficient
$\beta(I)$ for $I = (i_1, \ldots, i_n) \in [d]^n$ corresponding to the
monomial $x_1^{i_1} \cdots x_n^{i_n}$. 
\[ \beta(I) = \sum_{i=1}^m \alpha_i^* \sum_{\substack{k_1, \ldots, k_j \in [n]^j
\\ j \in \{0, \ldots, d\} \\ M(k_1, \ldots, k_j ) = (i_1, \ldots, i_n)}}
(\vx_i)_{k_1} \cdots (\vx_i)_{k_j} = \sum_{i=1}^m \alpha_i^* C\left(i_1, \ldots,
i_n\right) (\vx_i)_{1}^{i_1} \cdots (\vx_i)_{n}^{i_n} \]
Observe that the above can be easily computed since we know $\vx_i$ for all $i
\in [m]$, and the function $C$ can be efficiently computed as discussed before using the
multinomial theorem. Hence, the hypothesis is itself a polynomial of degree at most $d$,
any desired coefficient of which can be computed efficiently. 

\subsection{Formal Version of Theorem~\ref{thm:poly} and Its Proof}
The rest of this section is devoted to the proof of Theorem~\ref{thm:poly} (or, more precisely,
its formal variant Theorem \ref{thm:formalpoly} below, which makes explicit the conditions on the loss
function $\ell$ that are required for the theorem to hold). In
particular, we show that whenever the sample size $m$ is a sufficiently large
polynomial in $d$, $n$, $B$, $1/\eps$, and $\log(1/\delta)$, %
%
%
the hypothesis $h$ output by the algorithm satisfies 
$$\underset{(\vx,y) \sim D}{\mathbb{E}}[\ell(h(\vx), y)] \leq
\mathsf{opt} + \epsilon.$$ where $\mathsf{opt}$ is the error of the
best fitting multivariate polynomial $p$ of degree $d$ whose sum of
squares of coefficients is bounded by $B$.


\begin{theorem}[Formal Version of Theorem \ref{thm:poly}]
	\label{thm:formalpoly}   \label{thm:noisypoly}
	Let $\Poly(n, d, B)$ be the class of polynomials $p \colon \S^{n-1} \to [-1, 1]$
	in $n$ variables such that that the total degree of $p$ is at most $d$, and
	the sum of squares of coefficients of $p$ (in the standard monomial basis)
	is at most $B$. Let $\ell$ be any loss function such that is convex,
	monotone, and $L$-Lipschitz and $b$-bounded in the interval $[-1, 1]$.  Then
	$\poly(n, d, B)$ is agnostically learnable under any (unknown) distribution over
	$\mathbb{S}^{n-1} \times [-1, 1]$ with respect to the loss function $\ell$
	in time $\mathsf{poly}(n, d, B, 1/\epsilon, L, b, \log\frac{1}{\delta})$. The learning algorithm
	is proper if the loss function $\ell$ equals $\ell_p$ for constant $p > 0$.\jtmnote{Are we happy that I
	inserted this last sentence into the theorem?} 
\end{theorem}

\begin{proof}
	In order to prove the theorem, we need to bound 
	\begin{align}
	\notag	\falsejt(h; D) &= \E_{(\vx, y) \sim D}[\ell(h(\vx), y)].
	\end{align}



	We know that for all $\vx \in \S^{n-1}$, $\langle \psi_d(\vx), \psi_d(\vx)
	\rangle = d+1$.  Moreover, letting $\vv_p$ be the corresponding element of
	the RKHS for polynomial $p \in \mathcal{C}$, we know from previous analysis
	that $\langle \vv_p, \vv_p \rangle \leq B$.  In addition, the function
	$\clip_{-1, 1} : \reals \rightarrow [-1, 1]$ satisfies $\clip_{-1, 1}(0) =
	0$, and $\clip_{-1, 1}$ is $1$-Lipschitz.  Thus,
	Theorems~\ref{rademachercomplexity} and \ref{rademachercomplexity2} imply
	the following:
	\begin{align}
		\calR_m(\mathcal{C}) &\leq  \sqrt{\frac{(d+1) \cdot B}{m}}, \label{rcbounddie} \\
		\calR_m(\clip_{-1, 1} \circ \mathcal{C}) &\leq 2 \cdot \sqrt{\frac{(d + 1) \cdot B }{m}}.
		\label{rcbound2die}
	\end{align} 

	By assumption, $\ell$ is $L$-Lipschitz in its first argument and $b$-bounded
	in the interval $[-1, 1]$.  We assume the following bound on $m$ (note that
	it is polynomial in all the required factors):
	\begin{align}
		m &\geq \frac{1}{\epsilon^2} \left( 8 \mathsf{max}\{L, \epsilon^{-1} \} \sqrt{(d + 1) \cdot
		B} + \mathsf{max}\{b, 1\} \cdot \sqrt{2 \log \frac{1}{\delta}} \right)^2.
		\label{eqn:m-bounddie}
	\end{align}

	In the rest of the proof we assume that for every $f \in \Poly(n, d, B)$,
	the following hold:
	\begin{align}
		|\falsejt(f; D) - \hat{\falsejt}(f; S)| &\leq \epsilon. \label{eqn:jtl} \\
		|\falsejt(\clip_{-1, 1} \circ f; D) - \hat{\falsejt}(\clip_{-1, 1} \circ f; S)| &\leq \epsilon. \label{eqn:jtlclip}
	\end{align}
	Using Theorem~\ref{generalizationbound} together with the bounds on
	Rademacher complexity given by~\eqref{rcbounddie} and~\eqref{rcbound2die}
	and the $L$-Lipschitz continuity in its first argument and $b$-boundedness
	of $\ell$ on the interval $[-1, 1]$, we get that the above inequalities hold
	with probability at least $1 - 2\delta$. We let the algorithm fail with
	probability $2 \delta$.

	Now consider the following to bound $\falsejt(h; D)$.  Letting $p$ be any
	polynomial in $\Poly(n, d, B)$,
	\begin{align}
		\falsejt(h; D) & \leq \hat{\falsejt}(h; S) + \epsilon \label{eqn:fff1} \\
		& \leq \hat{\falsejt}(f; S) + \epsilon \label{eqn:fff2} \\
		& \leq \hat{\falsejt}(p; S) + \epsilon \label{eqn:fff3} \\
		& \leq \falsejt(p; D) + 2 \cdot \epsilon \label{eqn:fff4}
	\end{align}
	Above in~(\ref{eqn:fff1}), we appeal to~\eqref{eqn:jtlclip}.
	In~(\ref{eqn:fff2}), we use the fact that $D$ is a distribution over $\S^{n-1} \times [-1, 1]$,
	and $\ell$ is monotone. In~(\ref{eqn:fff3}), we use the fact that the coefficient vector of $p$
	is a feasible solution to Optimization Problem \ref{alg:optprobdie}, and Optimization
	Problem \ref{alg:finaldie} is a reformulation of 
	Optimization Problem \ref{alg:optprobdie}. Finally, in \eqref{eqn:fff4},
	we appeal~\eqref{eqn:jtl}. 
	
	The theorem now follows by replacing $\epsilon$ with $\epsilon/2$, $\delta$
	with $\delta/2$, and observing that the algorithm runs in time $\poly(m) =
	\mathsf{poly}(n, d, B, 1/\epsilon, L, b, \log\frac{1}{\delta})$.
\end{proof}

\section{Networks of ReLUs}
\label{sec:networkrelu}
In this section, we extend learnability results for a single ReLU to network
of ReLUs.  Our results in this section apply to the standard agnostic model of
learning in the case that the output is a linear combination of hidden units.
If our output layer, however, is a single ReLU, then our results can be
extended to the reliable setting using similar techniques from Section
\ref{sec:relu}.

We will use the same framework as Zhang~\etal~\cite{Zhang}, who showed
how to learn networks where the activation function is computed {\em
  exactly} by a power series (with bounded sum of squares of
coefficients $B$) with respect to loss functions that are bounded on
a domain that is a function of $B$.  Their algorithm works by repeatedly
composing the kernel of Shalev-Shwartz et al. \cite{SSSS} and
optimizing in the corresponding RKHS.

Note, however, that since $\rec$ is not differentiable at $0$, there is no
power series for $\rec$, and the approach of Zhang~\etal~\cite{Zhang}
cannot be used; their work applies to a smooth activation function
that has a shape that is ``Sigmoid-like''  or ``ReLU-like,'' but is not a good
approximation to $\rec$ in a precise mathematical sense.

We generalize their results to activation functions that are {\em
  approximated} by polynomials.  This allows us to capture many
classes of activation functions including ReLUs.  Our clipping
technique also allows us to work with respect to a broader class of
loss functions.


Our results for learning networks of ReLUs have a number of new applications.  First, we give the first efficient algorithms for learning
``parameterized'' ReLUs and ``leaky'' ReLUs.  Second, we obtain the first
polynomial-time approximation schemes for convex piecewise-linear regression
(see Section \ref{sec:convexpiecewise} for details).  As far as we are aware,
there were no provably efficient algorithms known for these types of
multivariate piecewise-linear regression problems.




\subsection{Notation}

We use the following notation of Zhang~\etal~\cite{Zhang}. Consider a network
with $D$ hidden layers and an output unit (we assume that the output is
one-dimensional). Let $\sigma: \mathbb{R} \rightarrow \mathbb{R}$ denote the
activation function applied at each unit of all the hidden layers. Let
$n^{(i)}$ denote the number of units in hidden layer $i$ with $n^{(0)} = n$
(i.e., input dimension) and $w^{(i)}_{jk}$ be the weight of the edge between
unit $j$ in layer $i$ and unit $k$ in layer $i+1$. We define, $y^{(i)}_j$ to be
the function that maps $\vx \in \mathcal{X}$ to the output of unit $j$ in layer
$i$,
\[
y_j^{(i)}(\vx) = \sigma\left(\sum_{k=1}^{n^{(i-1)}} w_{jk}^{(i-1)} \cdot y_k^{(i-1)}(\vx) \right),
\]
where $y^{(0)}_j(\vx) = \vx$ for all $j$. We similarly define $h^{(i)}_j$ to be
the function that maps $\vx \in \mathcal{X}$ to the input of unit $j$ in layer
$i+1$:
\[
h^{(i)}_j(\vx) = \sum_{k=1}^{n^{(i)}} w^{(i)}_{jk} \cdot y_k^{(i)}(\vx).
\]
Finally, we define the output of the network as a function $\mathcal{N}: \mathbb{R}^n \rightarrow \mathbb{R}$ as
\[
\mathcal{N}(\vx) =  \sum_{k=1}^{n^{(D)}} w_{1k}^{(D)} \cdot y_k^{(D)} (\vx).
\]
For a better understanding of the above notation, consider a fully-connected network $\mathcal{N}_1$ with a single hidden layer (these are also known as depth-2 networks) consisting of $k$ units:
\[
\mathcal{N}_1: \vx \mapsto \sum_{i=1}^k u_i \sigma(\vw_i \cdot \vx).
\]
In this case, output of unit $i \in [k]$ in the hidden layer is $y^{(1)}_i(\vx) = \sigma(\vw_i \cdot \vx)$ and the input to the same unit is $h^{(0)}_i(\vx) = \vw_i \cdot \vx$.

We consider a class of networks with edge weights of bounded $\ell_1$ or $\ell_2$ norm. The class is formalized as follows.

\begin{definition}[Weight-bounded Networks]
	Let $\mathcal{N}[\sigma, D, W, M]$ be the class of fully-connected networks
	with $D$ hidden layers and $\sigma$ as the activation function.
	Additionally, the weights are constrained such that $\sum_{j=1}^n
	(w_{ij}^{(0)})^2 \leq M^2$ for all units $i$ in layer 0 and
	$\sum_{k=1}^{n^{(i)}} |w_{jk}^{(i)}| \leq W$ for all units $j$ in all layers
	$i \in \{1, \ldots, D\}$. Also, the inputs to each unit are bounded in magnitude by $M$,
	\ie $h^{(l)}_j(\vx) \in [-M, M]$ with $M \geq1$ for each $l < D$ and $j=1, \dots, n^{(l+1)}$.
\end{definition}

We consider activation functions which can be approximated by polynomials with
sum of squares of coefficients bounded. We term them \textit{low-weight
approximable} activation functions, formalized as follows.



\begin{definition}[Low-weight Approximable Functions] \label{defn:lowwtpolyapprox}
	For activation function $\sigma : \reals \rightarrow \reals$, for $\epsilon
	\in (0, 1)$, $M \geq 1$, $B \geq 1$, we say that a polynomial $p(t) =
	\sum_{i = 1}^d \beta_i t^i$ is a degree $d$, $(\epsilon, M, B)$-approximation
	to $\sigma$ if for every $t \in [-M, M]$, $|\sigma(t) - p(t)| \leq \epsilon$
	and furthermore, $\sum_{i=0}^d 2^i \beta_i^2 \leq B$.
\end{definition}



\subsection{Approximate Polynomial Networks}

We first bound the error incurred when each activation function is replaced by
a corresponding low-weight polynomial approximation.

\begin{theorem}[Approximate Polynomial Network]
\label{thm:poly_network}
	Let $\sigma$ be an activation function that is 1-Lipschitz%
	\footnote{Note that this is not a restriction, as we have not explicitly
	constrained the weights $W$. Thus, to allow a Lipschitz constant $L$, we
	simply replace $W$ by $WL$. \label{ftn:Lipschitz}}
	and such that there exists a degree $d$ polynomial $p$ that is a
	$(\frac{\epsilon}{W^D D}, 2M, B)$ approximation for $\sigma$, with $\epsilon
	\in (0, 1)$, with $d, M, B \geq 1$. Then, for all $\mathcal{N} \in
	\mathcal{N}[\sigma, D, W, M]$, there exists $\bar{\mathcal{N}} \in
	\mathcal{N}[p, D, W, 2M]$ such that
	\[
		\sup_{\vx \in \S^{n-1}} \left|\mathcal{N}(\vx) -
		\bar{\mathcal{N}}(\vx)\right| \leq \epsilon.
	\]
\end{theorem}
\begin{proof}
	Let $\mathcal{N} \in \mathcal{N}[\sigma, D, W, M]$ and let $\bar{\mathcal{N}} \in \mathcal{N}[p, D, W, M]$ be such that it has the same structure and weights as $\mathcal{N}$ with the activation replaced with $p$. For $\mathcal{N}$ let $h^{(i)}(\vx)$ be the inputs to layer $i+1$ and $y^{(i)}(\vx)$ be the outputs of unit $j$ of layer $i$ as defined previously. Correspondingly, for $\bar{\mathcal{N}}$ let $\bar{h}^{(i)}(\vx)$ be the inputs to layer $i+1$ and $\bar{y}^{(i)}(\vx)$ be the outputs of layer $i$. We prove by induction on layer $i$ that for all units $j$ of layer $i$,
\begin{equation}
\label{eqn:input_approx}
\sup_{\vx \in \S^{n-1}} \left|h^{(i)}_j(\vx) - \bar{h}^{(i)}_j(\vx)\right| \leq \frac{i\epsilon}{W^{D-i}D}.
\end{equation}

For layer $i=0$, we have $h^{(0)}(\vx)_j = \bar{h}^{(0)}(\vx)_j = \textbf{w}_{j}^{(0)} \cdot \vx \in [-M,M]$ which trivially satisfies (\ref{eqn:input_approx}). Now, we prove that the desired property holds for layer $l$, assuming the following holds for layer $l-1$. We have for all units $j$ in layer $l-1$,
\begin{equation}
\label{eqn:inductionhypothesis}
\sup_{\vx \in \S^{n-1}} \left|h_{j}^{(l-1)}(\vx) - \bar{h}_j^{(l-1)}(\vx)\right| \leq \frac{(l-1)\epsilon}{W^{D-l+1}D}.
\end{equation}
Note that this implies that $\left|\bar{h}^{(l-1)}_j(\vx)\right| \leq \left| h^{(l-1)}_j(\vx)\right| + \frac{(l-1)\epsilon}{W^{D-l+1}D} \leq 2M$. Here the second inequality follows from the assumption that inputs to each unit are bounded by $M$ and $\epsilon < 1$. We have for all $\vx$ and $j$, 
\begin{align}
\left|h_{j}^{(l)}(\vx) - \bar{h}_j^{(l)}(\vx)\right| &= \left|\sum_{k=1}^{n^{(l)}} w^{(l)}_{jk} \cdot \sigma\left(h_k^{(l-1)}(\vx) \right) - \sum_{k=1}^{n^{(l)}} w^{(l)}_{jk } \cdot p\left(\bar{h}_k^{(l-1)}(\vx)  \right) \right| \nonumber \\
&= \sum_{k=1}^{n^{(l)}} \left|w^{(l)}_{jk} \right| \left| \sigma\left(h_k^{(l-1)}(\vx) \right) - p\left(\bar{h}_k^{(l-1)}(\vx)\right) \right| \nonumber \\
&\leq \sum_{k=1}^{n^{(l)}} \left|w^{(l)}_{jk}\right| \left( \left| \sigma\left(h_k^{(l-1)}(\vx) \right) - \sigma\left(\bar{h}_k^{(l-1)}(\vx)\right) \right| + \frac{\epsilon}{W^DD} \right) \label{eqn:replacebyapprox}\\
&\leq \sum_{k=1}^{n^{(l)}} \left|w^{(l)}_{jk} \right| \left( \left| h_k^{(l-1)}(\vx)  -\bar{h}_k^{(l-1)}(\vx) \right| +  \frac{\epsilon}{W^{D}D} \right) \label{eqn:useLipschitz}\\
&\leq \sum_{k=1}^{n^{(l)}} \left|w^{(l)}_{jk} \right| \left(\frac{(l-1)\epsilon}{W^{D-l+1}D} + \frac{\epsilon}{W^{D}D} \right) \label{eqn:inductionstep}\\
&= \|\textbf{w}_j^{(l)}\|_1\frac{l \cdot \epsilon}{W^{D-l+1}D} \nonumber \\
&\leq \frac{l \cdot \epsilon}{W^{D-l}D} \label{eqn:l1normbound}
\end{align}
Step (\ref{eqn:replacebyapprox}) follows since $\bar{h}^{(l-1)}_j(\vx) \in [-2M, 2M]$  and $p$ uniformly $\frac{\epsilon}{W^DD}$-approximates $\sigma$ in $[-2M, 2M]$. Step (\ref{eqn:useLipschitz}) follows from $\sigma$ being 1-Lipschitz. Step (\ref{eqn:inductionstep}) follows from (\ref{eqn:inductionhypothesis}). Finally Step (\ref{eqn:l1normbound}) follows from $\|\textbf{w}_j^{(l)}\|_1 \leq W$ which is given. This completes the inductive proof.

We conclude by noting that $\mathcal{N}(\vx) = h_{1}^{(D)}(\vx)$ and $\bar{\mathcal{N}}(\vx) = \bar{h}_{1}^{(D)}(\vx)$. Thus, from above we get,
\[
\sup_{\vx \in \S^{n-1}} \left|\mathcal{N}(\vx) - \bar{\mathcal{N}}(\vx)\right| = \sup_{\vx \in \S^{n-1}} \left|h_{1}^{(N)}(\vx) - \bar{h}_1^{(N)}(\vx)\right| \leq  \epsilon.
\]
This completes the proof.
\end{proof}

Given the above transformation to a polynomial network and associated
error bounds, we apply the main theorem of Zhang et al. \cite{Zhang}
combined with the clipping technique from Section \ref{sec:relu} to obtain the following result:

\begin{theorem}[Learnability of Neural Network]
\label{thm:learn_network}
	Let $\sigma$ be an activation function that is
	1-Lipschitz\footref{ftn:Lipschitz} and such that there exists a degree $d$
	polynomial $p$ that is an $(\frac{\epsilon}{(L + 1)\cdot W^D \cdot D}, 2M,
	B)$ approximation for $\sigma$, for $d, B, M \geq 1$. 
	Let $\ell$ be a loss function that is convex, $L$-Lipschitz in the first
	argument, and $b$ bounded on $[-2M \cdot W, 2M \cdot W]$.
	Then there exists an
	algorithm that outputs a predictor $\hat{f}$ such that with probability at
	least $1-\delta$, for any (unknown) distribution $\D$ over $\S^{n-1} \times
	[-M \cdot W, M \cdot W]$,
	\[
		\E_{(\vx,y) \sim \D}[\ell(\hat{f}(\vx), y)] \leq
		\min_{\mathcal{N} \in \mathcal{N}[\sigma, D, W, M]} \E_{(\vx,y) \sim
		\D}[\ell(\mathcal{N}(\vx), y)] + \epsilon.
	\]
	The time complexity of the above algorithm is bounded by $n^{O(1)} \cdot
	B^{O(d)^{D-1}} \cdot \log(1/\delta)$, where $d$ is the degree of $p$, and
	$B$ is a bound on $\sum_{i=0}^d 2^i \beta_i^2$ (see
	Defn.~\ref{defn:lowwtpolyapprox}).
	\sgmnote{Varun: Are we adding dependence on $b$?}
\end{theorem}
\begin{proof}
	From Theorem \ref{thm:poly_network} we have that for all $\mathcal{N} \in \mathcal{N}[\sigma, D, W, M]$, there is a network $\bar{\mathcal{N}} \in \mathcal{N}[p, D, W, M]$ such that
\[
\sup_{\vx \in \S^{n-1}} \left|\mathcal{N}(\vx) - \bar{\mathcal{N}}(\vx)\right| \leq \frac{\epsilon}{L+1}.
\]

Since the loss function $\ell$ is $L$-Lipschitz, this implies that
\begin{equation}
\label{eqn:loss_err}
\ell(\bar{\mathcal{N}}(x), y)  - \ell(\mathcal{N}(x), y)  \leq L \cdot |\bar{\mathcal{N}}(x) - \mathcal{N}(x)| \leq  \frac{L}{L+1} \cdot \epsilon.
\end{equation}

Let $\mathcal{N}_{\min} =  \operatorname*{arg\,min}_{\mathcal{N} \in \mathcal{N}[\sigma, D, W, M]} \E_{(\vx,y) \sim \D}[\ell(\mathcal{N}(\vx), y)]$. By the above, we get that there exists $\bar{\mathcal{N}}_{\min} \in \mathcal{N}[p, D, W, M]$ such that
\begin{align*}
 \min_{\bar{\mathcal{N}} \in \mathcal{N}[p, D, W, M]} \E_{(\vx,y) \sim \D}[\ell(\bar{\mathcal{N}}(\vx), y)] & \leq  \E_{(x,y) \sim \D}[\ell(\bar{\mathcal{N}}_{\min}(\vx), y)]  \\
 & \leq \E_{(\vx,y) \sim \D}[\ell(\mathcal{N}_{\min}(\vx), y)] + L \cdot \epsilon  \\
 & = \min_{\mathcal{N} \in \mathcal{N}[\sigma, D, W, M]} \E_{(\vx,y) \sim \D}[\ell(\mathcal{N}(\vx), y)] + \frac{L}{L+1} \cdot \epsilon.
\end{align*}

Now from \cite[Theorem 1]{Zhang}, we know that there exists an algorithm that outputs a predictor $\hat{f}$ such that with probability at least $1-\delta$ for any distribution $\D$ 
\begin{equation*}
\E_{(\vx,y) \sim \D}[\ell(\hat{f}(\vx), y)] \leq \min_{\bar{\mathcal{N}} \in \mathcal{N}[p, D, W, M]} \E_{(\vx,y) \sim \D}[\ell(\bar{\mathcal{N}}(\vx), y)] + \frac{\epsilon}{L+1}.
\end{equation*}

For loss functions that take on large values on the range of the
predictor, we instead output the clipped version of the predictor
$\mathsf{clip}(\hat{f})$ in order to satisfy the requirements of the
Rademacher bounds (as in Section \ref{sec:relu}).

The runtime of the algorithm is $\poly(n, (L+1)/\epsilon, \log(1/\delta), H^D(1))$,  where $H(a) = \sqrt{\sum_{i=0}^d 2^i\beta_ia^{2i}}$, and $H^{(D)}$ is obtained by composing $H$ with itself $D$ times. 
By simple algebra, we conclude that $H^D(1)$ is bounded by $B^{O(d)^{D - 1}}$.

Combining the above inequalities, we have 
\[
\E_{(\vx,y) \sim \D}[\ell(\hat{f}(\vx), y)] \leq \min_{\mathcal{N} \in \mathcal{N}[D, W, M, \rec]} \E_{(\vx,y) \sim \D}[\ell(\mathcal{N}(\vx), y)] + \epsilon.
\]
This completes the proof.
\end{proof}

We can now state the learnability result for ReLU networks as follows.
\begin{corollary}[Learnability of ReLU Network]
\label{thm:learn_relu_network}
There exists an algorithm that outputs a predictor $\hat{f}$ such that with probability at least $1-\delta$ for any distribution $\D$ over $\S^{n-1} \times [-M \cdot W, M \cdot W]$, and loss function $\ell$ which is convex, $L$-Lipschitz in the first argument, and $b$ bounded on $[-2M \cdot W, 2M \cdot W]$,
\[
\E_{(\vx,y) \sim \D}[\ell(\hat{f}(\vx), y)] \leq \min_{\mathcal{N} \in \mathcal{N}[D, W, M, \relu]} \E_{(\vx,y) \sim \D}[\ell(\mathcal{N}(\vx), y)] + \epsilon.
\]
The time complexity of the above algorithm is bounded by $n^{O(1)} \cdot 2^{((L+1)\cdot M \cdot W^{D} \cdot D \cdot \epsilon^{-1})^D} \cdot \log(1/\delta)$.
\end{corollary}
The proof of the corollary follows from applying Theorem \ref{thm:learn_network} for the activation function $\rec$ since $\rec$ is 1-Lipschitz and \textit{low-weight approximable} (from Theorem \ref{reluapprox} and \ref{coeffboundReLU}). We obtain the following corollary specifically for depth-2 networks.
\begin{corollary}
Depth-2 networks with $k$ hidden units and activation function $\rec$ such that the weight vectors have $\ell_2$-norm bounded by 1 are agnostically learnable over $\S^{n-1} \times [-\sqrt{k}, \sqrt{k}]$ with respect to loss function $\ell$ which is convex, $O(1)$-Lipschitz in the first argument, and $b$ bounded on $[-2\sqrt{k}, 2\sqrt{k}]$ in time $n^{O(1)} \cdot 2^{O(\sqrt{k}/\epsilon)} \cdot \log(1/\delta)$.
\end{corollary}
The proof of the corollary follows from setting $L=1$, $D=1$, $M = 1$ and $ W = \sqrt{k}$ in Theorem \ref{thm:learn_relu_network}. $W = \sqrt{k}$ follows from bounding the $\ell_1$-norm of the weights given a bound on the $\ell_2$-norm.

We remark here that the above analysis holds for fully-connected networks with activation function $\sig(x) = \frac{1}{1 + e^{-x}}$ (Sigmoid function). Note that $\sig$ is 1-Lipschitz. The following lemma due to Livni~\etal~\cite[Lemma 2]{LivniSS14} exhibits a low degree polynomial approximation for $\sig$. It is in turn based on a result of Shalev-Shwartz~\etal~\cite[Lemma 2]{SSSS}.
\begin{lemma}[Livni et al. \cite{LivniSS14}]
\label{sigapprox}
For $\epsilon \in (0,1)$, there exists a polynomial $p(a) = \sum_{i=1}^d \beta_i a^i$ for $d = O(\log(1/\epsilon))$ such that for all $a \in [-1,1]$, $|p(a) - \sig(a)| \leq \epsilon$.
\end{lemma}
Let $p(a) = \sum_{i=1}^d \beta_i a^i$ be the uniform $\epsilon$-approximation $\sig$ which is guaranteed to exist by the above lemma. Using a similar trick as in Lemma \ref{reluapprox}, we can further bound $p([-1, 1]) \subseteq [0, 1]$. Also, using Lemma \ref{coeffboundReLU}, we can show that $\sum_{i=0}^d 2^i\beta_i^2$ is bounded by $(1/\epsilon)^{O(1)}$.  This shows that $\sig$ is \textit{low-weight approximable}.

Using Theorem \ref{thm:learn_network}, we state the following learnability result for depth-2 sigmoid networks.
\begin{corollary} \label{oneofflabel}
Depth-2 networks with $k$ hidden units and sigmoidal activation function such that the weight vectors have $\ell_2$-norm bounded by 1 are agnostically learnable over $\S^{n-1} \times [-\sqrt{k}, \sqrt{k}]$ with respect to loss function $\ell$ which is convex, $O(1)$-Lipschitz in the first argument, and $b$ bounded on $[-2\sqrt{k}, 2\sqrt{k}]$ in time $\poly(n, k , 1/\epsilon, \log(1/\delta))$.
\end{corollary}

Observe that the above result is polynomial in all parameters.
Livni~\etal (cf. \cite[Theorem 5]{LivniSS14}) state an incomparable
result for learning sigmoids: their runtime is superpolynomial in $n$
for $L = \omega(1)$, where $L$ is the bound on $\ell_1$-norm of the
weight vectors ($L$ may be as large as $\sqrt{k}$ in the setting of
Corollary \ref{oneofflabel}).  They, however, work over the Boolean
cube (whereas we are working over the domain $\S^{n-1}$). 

\subsection{Application: Learning Parametric Rectified Linear Unit} \label{sec:prelu}
A Parametric Rectified Linear Unit (PReLU) is a generalization of ReLU introduced by He~\etal~\cite{he2015delving}. Compared to the ReLU, it has an additional parameter that is learned. Formally, it is defined as
\begin{definition}[Parametric Rectifier]
	The parametric rectifier (denoted by $\sigma_{\text{PReLU}}$) is an activation function defined as
	\[
	\sigma_{\text{PReLU}}(x) = \begin{cases}
		x & \text{if } x \geq 0 \\
		a \cdot x & \text{if } x < 0
			\end{cases}
	\]
	where $a$ is a learnable parameter.
\end{definition}
Note that we can represent $\sigma_{\text{PReLU}}(x) = \mathsf{max}(0, x) - a \cdot \mathsf{max}(0, -x) = \rec(x) - a \cdot \rec(-x)$ which is a depth-2 network of ReLUs. Therefore, we can state the following learnability result for a single PReLU parameterized by a weight vector $\vw$ based on learning depth-2 ReLU networks.
\begin{corollary}
Let PReLU with the parameter $a$ be such that $|a|$ is bounded by a constant and the weight vector $\vw$ has 2-norm bounded by 1. Then, PReLU is agnostically learnable over $\S^{n-1}$ with respect to any $O(1)$-Lipschitz loss function in time $n^{O(1)} \cdot 2^{O(1/\epsilon)} \cdot \log(1/\delta)$.
\end{corollary}
The proof of the corollary follows from setting $L=1$, $D=1$, $M = 1$ and $ W = O(1)$ in Theorem \ref{thm:learn_relu_network}.

The condition that $|a|$ be bounded by 1 is reasonable as in practice the value of $a$ is very rarely above 1 as observed by He~\etal~\cite{he2015delving}. Also note that Leaky-ReLUs \cite{maas2013rectifier} are PReLUs with fixed $a$ (usually 0.01). Hence, we can agnostically learn them under the same conditions using an identical argument as above. Note that a network of PReLU can also be similarly learned as a ReLU by replacing each ReLU in the network by a linear combination of two ReLUs as described before.

\subsection{Application: Learning the Piecewise Linear Transfer Function} \label{sec:pw}
Several functions have been used to relax the $0/1$ loss in the
context of learning linear classifiers.  The best example is the
sigmoid function discussed earlier.  Here we consider the piecewise linear transfer function. Formally, it is defined as
\begin{definition}[Piecewise Linear Transfer Function]
	The $C$-Lipschitz piecewise linear transfer function (denoted by $\sigma_{\text{pw}}$) is an activation function defined as
	\[
	\sigma_{\text{pw}}(x) = \mathsf{max} \left(0, \mathsf{min} \left( \frac{1}{2} + Cx, 1 \right) \right).
	\]
\end{definition}
Note that we can represent $\sigma_{\text{pw}}(x) = \mathsf{max}\left(0, \frac{1}{2} + Cx \right) - \mathsf{max}\left(0, -\frac{1}{2} + Cx\right) = \rec\left(\frac{1}{2} + Cx\right)  - \rec\left(-\frac{1}{2} + Cx\right) $ which is a depth-2 network of ReLU. Therefore, we can state the following learnability result for a piecewise linear transfer function parameterized by weight vector $\vw$ following a similar argument as in the previous section.
\begin{corollary}
The class of $C$-Lipschitz piecewise linear transfer functions parametrized by weight vector $\vw$ with 2-norm bounded by 1 is agnostically learnable over $\S^{n-1}$ with respect to any $O(1)$-Lipschitz loss function in time $n^{O(1)} \cdot 2^{O(C/\epsilon)} \cdot \log(1/\delta)$.
\end{corollary}
The proof of the corollary follows from setting $L=1$, $D=1$, $M = 1$ and $ W = O(L)$ in Theorem \ref{thm:learn_relu_network}.

Shalev-Shwartz~\etal~\cite{SSSS} in Appendix A solved the above problem for $l_1$ loss and gave a running time with dependence on $C, \epsilon$ as $\poly\left(\exp\left(\frac{C^2}{\epsilon^2}\log \left(\frac{C}{\epsilon}\right)\right)\right)$. Our approach gives an exponential improvement in terms of $\frac{C}{\epsilon}$ and works for general constant Lipschitz loss functions.

\subsection{Application: Convex Piecewise-Linear Fitting}
\label{sec:convexpiecewise}
In this section we can use our learnability results for networks of
ReLUs to give polynomial-time approximation schemes for convex
piecewise-linear regression \cite{Mag}.  These problems have been studied in
optimization and notably in machine learning in the context of
Multivariate Adaptive Regression Splines (MARS \cite{MARS}). Note that
these are {\em not} the same as {\em univariate} piecewise or segmented regression
problems, for which polynomial-time algorithms are known.  Although
our algorithms run in time exponential in $k$ (the number of affine
functions), we note that no provably efficient algorithms were known prior to our work even for
the case $k = 2$.\footnote{Boyd and Magnani \cite{Mag} specifically focus on the case of
small $k$, writing ``Our interest, however, is in the case when the number of terms $k$ is relatively small, say no more than 10,
or a few 10s.''}

The key idea will be to reduce piecewise regression problems to an optimization
problem on networks of ReLUs using simple ReLU ``gadgets.''  We formally
describe the problems and describe the gadgets in detail. 



\subsubsection{Sum of Max 2-Affine}
We start with a simple class of convex piecewise linear functions represented as a sum of a fixed number of functions where each of these functions is a maximum of 2 affine functions. This is formally defined as follows.
\begin{definition} [Sum of $k$ Max 2-Affine Fitting \cite{Mag}]
\label{sum-max}
Let ${\cal C}$ be the class of functions of the form $f(x) =
\sum_{i=1}^k \mathsf{max}(\vw_{2i - 1} \cdot \vx, \vw_{2i} \cdot \vx)$ with $\vw_{1}, \ldots, \vw_{2k} \in \S^{n-1}$ mapping $\S^{n-1}$
to $\R$.  Let $\D$ be an (unknown) distribution on $\S^{n-1} \times [-k,k]$.  Given
i.i.d. examples drawn from $\D$, for any $\eps \in (0,1)$ find a function $h$ (not necessarily in {\cal C}) such that 
$\E_{(\vx,y) \sim \D}[ (h(\vx) - y)^2] \leq \min_{c \in {\cal C}} \E_{(\vx,y) \sim \D}[ (c(\vx) - y)^2] + \eps$.   
\end{definition}

It is easy to see that $\mathsf{max}(a, b) = \mathsf{max}(0, a-b) + \mathsf{max}(0, b) - \mathsf{max}(0, -b) = \rec(a-b) + \rec(b) - \rec(-b)$ where $\rec(a) = \mathsf{max}(0,a)$. This is simply a linear combination of ReLUs. We can thus represent $\mathsf{max}(\vw_1 \cdot \vx, \vw_2 \cdot \vx)$ as a depth-2 network (see Figure \ref{max-relu}).
\begin{figure}
\centering
\begin{tikzpicture}[shorten >=1pt,->,draw=black, node distance=3cm]
    \tikzstyle{unit}=[circle,minimum size=18pt]
    \tikzstyle{input unit}=[unit, fill=green!50!white];
    \tikzstyle{max unit}=[unit, fill=red!50!white];
    \tikzstyle{output unit}=[unit, fill=blue!50!white];
    \tikzstyle{layer} = [text width=4em, text centered]
    \foreach \name / \y in {1,...,2}
        \node[input unit] (I-\name) at (0,-\y*1.5) {$\vw_\y \cdot \vx$};
    \foreach \name / \y in {1,...,3}
        \path[yshift=0.75cm]
            node[max unit] (H1-\name) at (3cm,-\y*1.5 cm) {$\rec$};
    \node[output unit,pin={[pin edge={->}]right:$output$}] (O) at (6cm, -2.5cm) {};
\path[line width=0.3mm]
    (I-1) edge[black] (H1-1)
    (I-2) edge[dashed] (H1-1)
    (I-2) edge[black] (H1-2) 
    (I-2) edge[dashed] (H1-3) 
     
     (H1-1) edge[black] (O) 
    (H1-2) edge[black] (O) 
    (H1-3) edge[dashed] (O) ;     
\end{tikzpicture}
\caption{Representation of $\mathsf{max}(\vw_1 \cdot \vx, \vw_2 \cdot \vx)$ as a depth-2 ReLU network. Note that solid edges represent a weight of 1, dashed edges represent a weight of -1, and the absence of an edge represents a weight of 0.}
\label{max-relu}
\end{figure}
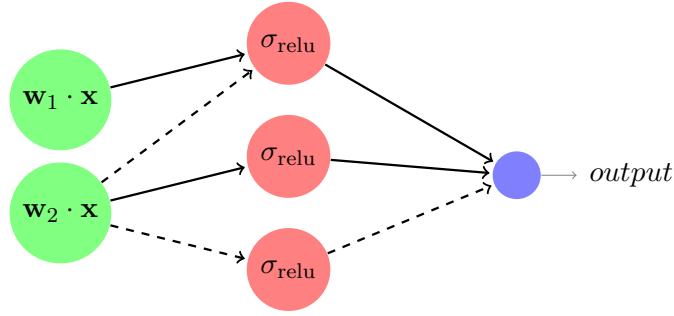
Adding copies of this, we can represent a sum of $k$ max 2-affine functions as a depth-2 network $\mathcal{N}_{\Sigma}$ with $3k$ hidden units and activation function $\rec$ satisfying the following properties,
\begin{itemize}
\item $\|\vw_j^{(0)}\| \leq 2$
\item $\|\vw_1^{(1)}\|_1 \leq 3k $
\item Each input to each unit is bounded in magnitude by 2.
\end{itemize}
The first property holds as $\|\vw_j^{(0)}\| \leq \mathsf{max}(\|\vw_{2j - 1} - \vw_{2j}\|, \|\vw_{2j - 1}\|, \|\vw_{2j}\|) \leq  \|\vw_{2j - 1}\| + \|\vw_{2j}\| \leq 2$ using the triangle inequality. The second holds because each of the $k$ max sub-networks contributes 3 to $\|\vw_1^{(1)}\|_1$. The third is implied by the fact that each input to each unit is bounded by $|\mathsf{max}(\vw_1 \cdot \vx, -\vw_1 \cdot \vx, (\vw_1 - \vw_2)  \cdot \vx)| \leq 2$.
\begin{theorem}
Let ${\cal C}$ be as in Definition \ref{sum-max}, then there is
an algorithm ${\cal A}$ for solving sum of $k$ max 2-affine
fitting problem in time $n^{O(1)}2^{O\left((k^2/\epsilon)\right)} \log(1/ \delta)$. 
\end{theorem}
\begin{proof}
As per our construction, we know that there exists a network $\mathcal{N}_{\Sigma}$ with activation function $\rec$ and 1 hidden layer such that $\|\vw_j^{(0)}\|_2 \leq 2$ and $\|\vw_1^{(1)}\|_1 \leq 3k$. Also, input to each unit is bounded in magnitude by 2. 
Thus, using Theorem \ref{thm:learn_relu_network} with $K =1$, $M=2$ and $W = 3k$ we get that there exists an algorithm that solves the sum of $k$ max 2-affine fitting problem in time $n^{O(1)} 2^{(O(k^2/\epsilon))} \log(1/\delta)$.
\end{proof}

\subsubsection{Max $k$-Affine}
In this section, we move to a more general convex piecewise linear function represented as the maximum of $k$ affine functions. This is formally defined as follows.
\begin{definition} [Max $k$-Affine Fitting \cite{Mag}]
\label{max-k}
Let ${\cal C}$ be the class of functions of the form $f(x) =
\mathsf{max}(\vw_{1} \cdot \vx, \ldots, \vw_{k} \cdot \vx)$ with $\vw_{1}, \ldots, \vw_{n} \in \S^{n-1}$ mapping $\S^{n-1}$
to $\R$.  Let $\D$ be a distribution on $\S^{n-1} \times [-1,1]$.  Given
i.i.d. examples drawn from $\D$, for any $\eps \in (0,1)$ find a function $h$ (not necessarily in {\cal C}) such that 
$\E_{(\vx,y) \sim \D}[ (h(\vx) - y)^2] \leq \min_{c \in {\cal C}} \E_{(\vx,y) \sim \D}[ (c(\vx) - y)^2] + \eps$.    
\end{definition}
Note that this form is universal since any convex piecewise-linear function can be expressed as a max-affine function, for some value of $k$. However, we focus on bounded $k$ and give learnability bounds in terms of $k$.

Observe that max $k$-affine can be expressed in a complete binary tree structure of height $\lceil \log k \rceil$ with a $\mathsf{max}$ operation at each unit and $\vw_{i} \cdot \vx$ for $i \in [k]$ at the $k$ leaf units (for example, see Figure \ref{k-max-tree} 
Note that if $k$ is not a power of 2, then we can trivially add leaves with value $\vw_1 \cdot \vx$ and make it a complete tree.
\begin{figure}
\centering
\begin{tikzpicture}[shorten >=1pt,->,draw=black, node distance=3cm]
    \tikzstyle{unit}=[circle,minimum size=18pt]
    \tikzstyle{input unit}=[unit, fill=green!50!white];
    \tikzstyle{max unit}=[unit, fill=red!50!white];
    \tikzstyle{layer} = [text width=4em, text centered]
    \foreach \name / \y in {1,...,4}
        \node[input unit] (I-\name) at (0,-\y*1.5) {$\vw_\y \cdot \vx$};
    \foreach \name / \y in {1,...,2}
        \path[yshift=0.625cm,line width=0.3mm]
            node[max unit] (H-\name) at (3cm,-\y*3 cm) {max};
    \node[max unit] (O) at (6cm, -4cm) {max};
    \foreach \source in {1,...,2}
            \path[line width=0.3mm] (I-\source) edge (H-1);
    \foreach \source in {3,...,4}
            \path[line width=0.3mm] (I-\source) edge (H-2);
     \foreach \source in {1,...,2}
        \path[line width=0.3mm] (H-\source) edge (O);
\end{tikzpicture}
\caption{Tree structure for evaluating max $k$-affine with $k=4$.}
\label{k-max-tree}
\end{figure}
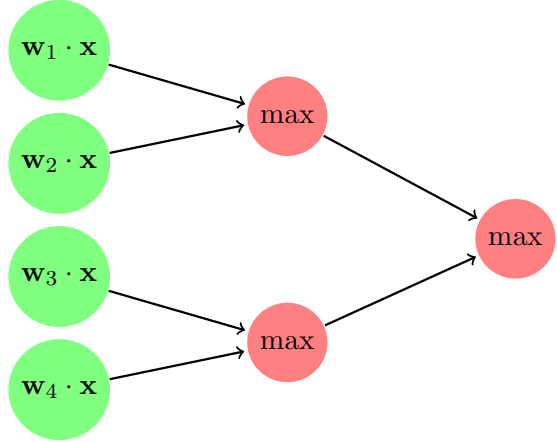

Thus, the class of convex piecewise linear functions can be expressed as a network of ReLUs with $\lceil \log k \rceil$ hidden layers by replacing each $\mathsf{max}$ unit in the tree by 3 ReLUs and adding an output unit. See Figure \ref{k-max-relu} 
for the construction for $k=4$.
\begin{figure}
\centering
\begin{tikzpicture}[shorten >=1pt,->,draw=black, node distance=3cm]
    \tikzstyle{unit}=[circle,minimum size=18pt]
    \tikzstyle{input unit}=[unit, fill=green!50!white];
    \tikzstyle{max unit}=[unit, fill=red!50!white];
    \tikzstyle{output unit}=[unit, fill=blue!50!white];
    \tikzstyle{layer} = [text width=4em, text centered]
    \foreach \name / \y in {1,...,2}
        \node[input unit] (I-\name) at (0,-\y*1.5) {$\vw_\y \cdot \vx$};
    \foreach \name / \y in {3,...,4}
        \node[input unit] (I-\name) at (0,-1.5-\y*1.5) {$\vw_\y \cdot \vx$};
    \foreach \name / \y in {1,...,6}
        \path[yshift=0.75cm]
            node[max unit] (H1-\name) at (3cm,-\y*1.5 cm) {$\rec$};
    \foreach \name / \y in {1,...,3}
        \path[yshift=1.5cm]
            node[max unit] (H2-\name) at (6cm,-\y*3 cm) {$\rec$};
    \node[output unit,pin={[pin edge={->}]right:$output$}] (O) at (9cm, -4.5cm) {};
\path[line width=0.3mm]
    (I-1) edge[black] (H1-1)
    (I-2) edge[dashed] (H1-1)
    (I-2) edge[black] (H1-2) 
    (I-2) edge[dashed] (H1-3) 
    (I-3) edge[black] (H1-4) 
    (I-4) edge[dashed] (H1-4) 
    (I-4) edge[black] (H1-5)
    (I-4) edge[dashed] (H1-6)
    
    (H1-1) edge[black] (H2-1) 
    (H1-2) edge[black] (H2-1) 
    (H1-3) edge[dashed] (H2-1) 
    (H1-4) edge[dashed] (H2-1) 
    (H1-5) edge[dashed] (H2-1) 
    (H1-6) edge[black] (H2-1) 
    (H1-4) edge[black] (H2-2) 
    (H1-5) edge[black] (H2-2) 
    (H1-6) edge[dashed] (H2-2) 
    (H1-4) edge[dashed] (H2-3) 
    (H1-5) edge[dashed] (H2-3) 
    (H1-6) edge[black] (H2-3) 
     
     (H2-1) edge[black] (O) 
    (H2-2) edge[black] (O) 
    (H2-3) edge[dashed] (O) ;     

\end{tikzpicture}
\caption{Network with $\rec$ for evaluating max $k$-affine with $k=4$. Note that solid edges represent a weight of 1, dashed edges represent a weight of -1 and the absence of an edge represents a weight of 0.}
\label{k-max-relu}
\end{figure}
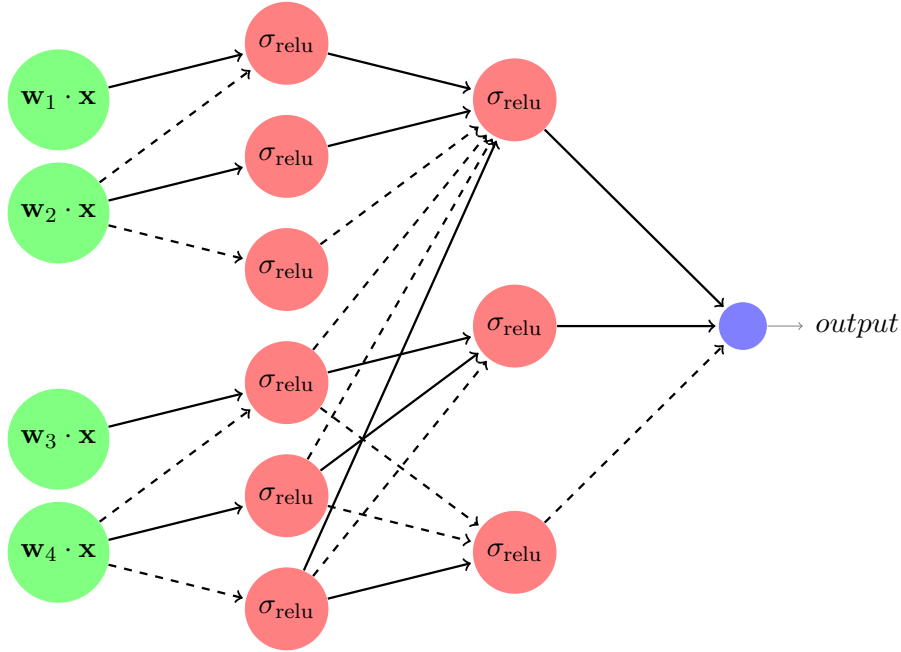

More formally, we have a network $\mathcal{N}_{\mathsf{max}}$ with $\lceil \log k \rceil$ hidden layers and one output unit with $\rec$ as the activation function. Hidden layer $i$ has $3 \cdot 2^{\lceil \log k \rceil-i}$ units. The weight vectors for the units in the first hidden layer are
\[
	\vw_{3j-m}^{(0)} = \begin{cases}
				\vw_{2j} - \vw_{2j-1} & m = 0 \\
				\vw_{2j-1} & m = 1 \\
				-\vw_{2j-1} & m = 2
			\end{cases}
\]
for $j \in [3\cdot 2^{\lceil \log k \rceil-1}]$. Further, the weight vectors input to hidden layer $i \in \{2, \ldots, \lceil \log k \rceil\}$ of the network are
\[
	\vw_{3j-m}^{(i-1)} = \begin{cases}
				\ve_{6j} + \ve_{6j-1} - \ve_{6j-2} - (\ve_{6j-3} + \ve_{6j-4} - \ve_{6j-5}) & m = 0 \\
				\ve_{6j-3} + \ve_{6j-4} - \ve_{6j-5} & m = 1 \\
				-(\ve_{6j-3} + \ve_{6j-4} - \ve_{6j-5})  & m = 2.
			\end{cases}
\]
for $j \in [3\cdot 2^{\lceil \log k \rceil - i + 1}]$. Note, $\ve_i$ refers to the vector with 1 at position $i$ and 0 everywhere else. Finally the weight vector for the output unit is $w_1^{(\lceil \log k \rceil)} = \ve_1 + \ve_2 - \ve_3$. The following properties of $\mathcal{N}_{\mathsf{max}}$ are easy to deduce.
\begin{itemize}
\item $\|\vw_j^{(0)}\|_2 \leq 2$
\item $\|\vw_j^{(i)}\|_1 \leq 6 $ for $i \in [\lceil \log k \rceil]$
\item The input to each unit is bounded by 2.
\end{itemize}
Here, the first and third conditions are the same conditions as in the previous section. The second holds by the values of the weights defined above. Using the above construction, we obtain the following result.

\begin{theorem}
Let ${\cal C}$ be as in Definition \ref{max-k}, then there is
an algorithm ${\cal A}$ for solving the max $k$-affine
fitting problem in time $n^{O(1)} \cdot 2^{O(k/\epsilon)^{\lceil \log k \rceil}} \cdot \log(1/ \delta)$. 
\end{theorem}
\begin{proof}
As per our construction, we know that there exists a network $\mathcal{N}_{\mathsf{max}}$ with activation function $\rec$ and $\lceil \log k \rceil$ hidden layers such that $\|\vw_j^{(0)}\|_2 \leq 2$ and $\|\vw_j^{(i)}\|_1 \leq 6 $ for $i \in [\lceil \log k \rceil]$. Also, input to each unit is bounded by 2. Thus, using Theorem \ref{thm:learn_relu_network} with $K =[\lceil \log k \rceil]$, $M=2$ and $W = 6$, we get that there exists an algorithm that solves the max $k$-affine problem in time $n^{O(1)} 2^{(k/\epsilon)^{O(\log k)}} \log(1/\delta)$.
\end{proof}

\section{Hardness of Learning ReLU}
\label{sec:hardness}
We also establish the first hardness results for learning a {\em
  single} ReLU with respect to distributions supported on the Boolean
hypercube ($\{0,1\}^n$).  The high-level ``takeaway'' from our hardness results is that learning
functions of the form $\mathsf{max}(0, \vw \cdot \vx)$ where $|\vw \cdot \vx|
\in \omega(1)$ is as hard as solving notoriously difficult problems in
computational learning theory.  This justifies our focus in previous
sections on input distributions supported on $\S^{n-1}$ and indicates
that learning real-valued functions on the sphere is one avenue for
{\em avoiding} the vast literature of hardness results on Boolean
function learning.

To begin, we recall the following problem from computational learning
theory widely thought to be computationally intractable.

\begin{definition} (Learning Sparse Parity with Noise) Let $\chi_{S}:
  \{0,1\}^n \rightarrow \{0,1\}^n$ be an unknown parity function on a
  subset $S$, $|S| \leq k$, of $n$ inputs bits (i.e., any input, restricted
  to $S$, with an odd number of ones is mapped to $1$ and $0$
  otherwise). Let ${\cal C}_k$ be the concept class of all parity
  functions on subsets $S$ of size at most k. Let $\D$ be a
  distribution on $\{0,1\}^n \times \{-1,1\}$ and define

\[ \mathsf{opt} = \min_{\chi \in {\cal C}_k} \Pr_{(\vx,y) \sim {\cal
    \D}}[\chi(\vx) \neq y]. \]

The {\em Sparse Learning Parity with Noise} problem is as follows:
Given i.i.d. examples drawn from $\D$, find $h$ such that
$\Pr_{(\vx,y) \sim \D}[h(\vx) \neq y] \leq \mathsf{opt} + \epsilon$.
\end{definition}

Our hardness assumption is as follows:

\begin{assumption}
 For every algorithm ${\cal A}$ that solves the Sparse Learning Parity
 with Noise problem, there exists $\epsilon = O(1)$ and $k
 \in \omega(1)$ such that ${\cal A}$ requires time $n^{\Omega(k)}$.
\end{assumption}

Any algorithm breaking the above assumption would be a major result in
theoretical computer science. The best known algorithms due to Blum,
Kalai, Wasserman \cite{BKW03} and Valiant \cite{Val15} run in time
$2^{O(n / \log n)}$ and $n^{0.8k}$, respectively. Under this
assumption, we can rule out polynomial-time algorithms for reliably
learning ReLUs on distributions supported on $\{0,1\}^n$.

\begin{theorem}
  Let ${\cal C}$ be the class of ReLUs over the domain ${\cal X} =
  \{0,1\}^n$ with the added restriction that $\|\vw\|_1 \leq 2k$.  Any
  algorithm ${\cal A}$ for reliably learning ${\cal C}$ in time $g(\epsilon)
  \cdot \mathsf{poly}(n)$ for any function $g$ will give a polynomial
  time algorithm for learning sparse parities with noise of size $k$
  for $\epsilon = O(1)$.
\end{theorem}

\begin{proof}

  We will show how to use a reliable ReLU learner to agnostically
  learn {\em conjunctions} on $\{0,1\}^n$ and use an observation due
  to Feldman and Kothari \cite{fk} who showed that agnostically
  learning conjunctions is harder than the Sparse Learning Parity with
  Noise problem. Let ${\cal CO}_k$ be the concept class of all Boolean
  conjunctions of length at most $k$. 

  Notice that for the domain ${\cal X} = \{0,1\}^n$, the conjunction
  of literals $x_{1},\ldots,x_{k}$ can be computed exactly as
  $\mathsf{max}(0,~x_1 + \cdots + x_k - (k-1))$.  Fix an arbitrary
  distribution $\D$ on $\{0,1\}^n \times \{0,1\}$ and define \[
  \mathsf{opt} = \min_{c \in {\cal CO}_k} \Pr_{(\vx,y) \sim {\D}}[c(\vx) \neq y]. \]

  Kalai et al. \cite{KalaiKMS2008} (Theorem 5) observed that in order output a
  hypothesis $h$ with error $\mathsf{opt} + \epsilon$ it suffices to
  minimize (to within $\epsilon$) the following quantity:

 \[ \mathsf{opt}_1 = \min_{c \in {\cal CO}_k} \E_{(\vx,y) \sim \D}[|c(\vx) - y]|. \]

  Consider the following transformed distribution ${\cal D'}$ on
  $\{0,1\}^n \times \{\epsilon, 1+\epsilon\}$ that adds a small
  positive $\epsilon$ to every $y$ output by $\D$.  Note that
  this changes $\mathsf{opt}_1$ by at most $\epsilon$.  Further, all
  labels in ${\cal D'}$ are now positive.  Since every $c \in
  {\cal CO}_k$ is computed exactly by a ReLU, and the reliable
  learning model demands that we minimize $\ellgtzero(h; {\cal D'})$ over all
  ReLUs, algorithm ${\cal A}$  will find an $h$ such that $\E_{(\vx,y) \sim {\cal
      D'}}[|h(\vx) - y|] \leq \mathsf{opt}_1 + \epsilon \leq
  \mathsf{opt} + 2\epsilon$.   By
  appropriately rescaling $\epsilon$, we have shown how to
  agnostically learn conjunctions using reliable learner ${\cal A}$.
  This completes the proof.
\end{proof}

The above proof also shows hardness of learning ReLUs agnostically.
Note the above hardness result holds if we require the learning
algorithm to succeed on all domains where $|(w \cdot x)|$ can grow
without bound with respect to $n$:

\begin{corollary}
  Let ${\cal A}$ be an algorithm that learns ReLUs on all domains
  ${\cal X} \subseteq \R^{n}$ where $(\vw \cdot \vx)$ may take on values
  that are $\omega(1)$ with respect to the dimension $n$.  Then any
  algorithm for reliably learning ${\cal C}$ in time $g(\eps) \cdot
  poly(n)$ will break the Sparse Learning Parity with Noise hardness
  assumption.
\end{corollary}

Finally, we point out Kalai et al. \cite{kkmrel} proved that reliably
learning conjunctions is also as hard as PAC Learning DNF formulas.
Thus, by our above reduction, any efficient algorithm for reliably
learning ReLUs would give an efficient algorithm for PAC learning DNF
formulas (again this would be considered a breakthrough result in
computational learning theory).

\section{Conclusions and Open Problems}
We have given the first set of efficient algorithms for ReLUs in a
natural learning model.  ReLUs are both effective in practice and,
unlike linear threshold functions (halfspaces), admit non-trivial
learning algorithms for {\em all} distributions with respect to
adversarial noise.  We ``sidestepped'' the hardness results in Boolean
function learning by focusing on problems that are not entirely
scale-invariant with respect to the choice of domain (e.g., reliably
learning ReLUs).  The obvious open question is to improve the
dependence of our main result on $1/\epsilon$.  We have no reason to
believe that $2^{O(1/\epsilon)}$ is the best possible.

\medskip \noindent \textbf{Acknowledgements.} The authors
are grateful to Sanjeev Arora and Roi Livni for helpful feedback 
and useful discussions on this work.


\bibliography{all-refs}

\begin{thebibliography}{10}

\bibitem{AndoniPV014}
Alexandr Andoni, Rina Panigrahy, Gregory Valiant, and Li~Zhang.
\newblock Learning sparse polynomial functions.
\newblock In {\em Proceedings of the Twenty-Fifth Annual {ACM-SIAM} Symposium
  on Discrete Algorithms, {SODA} 2014, Portland, Oregon, USA, January 5-7,
  2014}, pages 500--510, 2014.

\bibitem{rarora}
Raman Arora, Amitabh Basu, Poorya Mianjy, and Anribit Mukherjee.
\newblock Understanding deep neural networks with rectified linear units, 2016.
\newblock URL: \url{https://arxiv.org/abs/1611.01491}.

\bibitem{Bac:2014}
Francis Bach.
\newblock Breaking the curse of dimensionality with convex neural networks.
\newblock 2014.

\bibitem{BM:2002}
Peter~L. Bartlett and Shahar Mendelson.
\newblock Rademacher and gaussian complexities: Risk bounds and structural
  results.
\newblock {\em Journal of Machine Learning Research}, 3:463--482, 2002.

\bibitem{BKW03}
Blum, Kalai, and Wasserman.
\newblock Noise-tolerant learning, the parity problem, and the statistical
  query model.
\newblock {\em JACM: Journal of the ACM}, 50, 2003.

\bibitem{CS-T:2000}
Nello Cristianini and John Shawe-Taylor.
\newblock {\em An introduction to support vector machines and other
  kernel-based learning methods}.
\newblock Cambridge University Press, 2000.

\bibitem{Daniely}
Amit Daniely.
\newblock Complexity theoretic limitations on learning halfspaces.
\newblock In {\em STOC}, pages 105--117. ACM, 2016.

\bibitem{DKN10}
Ilias Diakonikolas, Daniel~M. Kane, and Jelani Nelson.
\newblock Bounded independence fools degree-2 threshold functions.
\newblock In {\em FOCS}, pages 11--20. IEEE Computer Society, 2010.

\bibitem{elden}
Ronen Eldan and Ohad Shamir.
\newblock The power of depth for feedforward neural networks.
\newblock In Vitaly Feldman, Alexander Rakhlin, and Ohad Shamir, editors, {\em
  Proceedings of the 29th Conference on Learning Theory, {COLT} 2016, New York,
  USA, June 23-26, 2016}, volume~49 of {\em {JMLR} Workshop and Conference
  Proceedings}, pages 907--940. JMLR.org, 2016.

\bibitem{nytimes}
Bassey Etim.
\newblock {A}pprove or {R}eject: {C}an {Y}ou {M}oderate {F}ive {N}ew {Y}ork
  {T}imes {C}omments?
\newblock {\em The New York Times}, 2016.
\newblock Originally published September 20, 2016. Retrieved October 4, 2016.

\bibitem{FGKP09}
V.~Feldman, P.~Gopalan, S.~Khot, and A.~K. Ponnuswami.
\newblock On agnostic learning of parities, monomials, and halfspaces.
\newblock {\em SIAM J. Comput}, 39(2):606--645, 2009.

\bibitem{fk}
Vitaly Feldman and Pravesh Kothari.
\newblock Agnostic learning of disjunctions on symmetric distributions.
\newblock {\em Journal of Machine Learning Research}, 16:3455--3467, 2015.

\bibitem{MARS}
Jerome~H. Friedman.
\newblock Multivariate adaptive regression splines.
\newblock {\em Ann. Statist}, 1991.

\bibitem{Haussler:1992}
David Haussler.
\newblock Decision theoretic generalizations of the pac model for neural net
  and other learning applications.
\newblock {\em Inf. Comput.}, 100(1):78--150, 1992.

\bibitem{he2015delving}
Kaiming He, Xiangyu Zhang, Shaoqing Ren, and Jian Sun.
\newblock Delving deep into rectifiers: Surpassing human-level performance on
  imagenet classification.
\newblock In {\em Proceedings of the IEEE International Conference on Computer
  Vision}, pages 1026--1034, 2015.

\bibitem{hofmann2008kernel}
Thomas Hofmann, Bernhard Sch{\"o}lkopf, and Alexander~J Smola.
\newblock Kernel methods in machine learning.
\newblock {\em The annals of statistics}, pages 1171--1220, 2008.

\bibitem{KST:2008}
Sham~M. Kakade, Karthik Sridharan, and Ambuj Tewari.
\newblock On the complexity of linear prediction: Risk bounds, margin bounds,
  and regularization.
\newblock 2008.

\bibitem{kkmrel}
Adam~Tauman Kalai, Varun Kanade, and Yishay Mansour.
\newblock Reliable agnostic learning.
\newblock {\em Journal of Computer and System Sciences}, 78(5):1481--1495,
  2012.

\bibitem{KalaiKMS2008}
Adam~Tauman Kalai, Adam~R. Klivans, Yishay Mansour, and Rocco~A. Servedio.
\newblock Agnostically learning halfspaces.
\newblock {\em SIAM Journal on Computing}, 37(6):1777--1805, 2008.

\bibitem{KSS:1994}
Michael~J. Kearns, Robert~E. Schapire, and Linda~M. Sellie.
\newblock Toward efficient agnostic learning.
\newblock {\em Mach. Learn.}, 17(2-3):115--141, 1994.

\bibitem{KS09}
A.~R. Klivans and A.~A. Sherstov.
\newblock Cryptographic hardness for learning intersections of halfspaces.
\newblock {\em J. Comput. Syst. Sci}, 75(1):2--12, 2009.

\bibitem{KK:2014}
Adam Klivans and Pravesh Kothari.
\newblock Embedding hard learning problems into gaussian space.
\newblock In {\em RANDOM}, 2014.

\bibitem{LeCun15}
Y.~LeCun, Y.~Bengio, and G.~Hinton.
\newblock Deep learning.
\newblock {\em Nature}, 521(7533):436--444, May 2015.

\bibitem{LT:1991}
Michel Ledoux and Michel Talagrand.
\newblock {\em Probability in Banach Spaces: {I}soperimetry and {P}rocesses}.
\newblock Springer, 1991.

\bibitem{LivniSS14}
Roi Livni, Shai Shalev{-}Shwartz, and Ohad Shamir.
\newblock On the computational efficiency of training neural networks.
\newblock pages 855--863, 2014.

\bibitem{maas2013rectifier}
Andrew~L Maas, Awni~Y Hannun, and Andrew~Y Ng.
\newblock Rectifier nonlinearities improve neural network acoustic models.
\newblock In {\em Proc. ICML}, volume~30, 2013.

\bibitem{Mag}
Alessandro Magnani and Stephen~P. Boyd.
\newblock Convex piecewise-linear fitting.
\newblock {\em Optimization and Engineering}, 10(1):1--17, 2009.

\bibitem{mercer}
James Mercer.
\newblock Functions of positive and negative type, and their connection with
  the theory of integral equations.
\newblock {\em Philosophical transactions of the royal society of London.
  Series A, containing papers of a mathematical or physical character},
  209:415--446, 1909.

\bibitem{New:1964}
D.~J. Newman.
\newblock Rational approximation to $|x|$.
\newblock {\em Michigan Math. J.}, 11(1):11--14, 03 1964.

\bibitem{Rig}
Phillippe Rigollet.
\newblock {\em High-Dimensional Statistics}.
\newblock MIT, 1st edition, 2015.

\bibitem{SSSS}
Shai Shalev-Shwartz, Ohad Shamir, and Karthik Sridharan.
\newblock Learning kernel-based halfspaces with the 0-1 loss.
\newblock {\em SIAM J. Comput.}, 40(6):1623--1646, 2011.

\bibitem{She:2012}
Alexander~A. Sherstov.
\newblock Making polynomials robust to noise.
\newblock In {\em Proceedings of the Forty-fourth Annual ACM Symposium on
  Theory of Computing}, STOC '12, pages 747--758, New York, NY, USA, 2012. ACM.

\bibitem{Val15}
Gregory Valiant.
\newblock Finding correlations in subquadratic time, with applications to
  learning parities and the closest pair problem.
\newblock {\em J. ACM}, 62(2):13:1--13:45, May 2015.

\bibitem{wiki_multi}
Wikipedia.
\newblock Multinomial theorem --- {W}ikipedia{,} the free encyclopedia, 2016.
\newblock URL: \url{https://en.wikipedia.org/wiki/Multinomial_theorem}.

\bibitem{wikipedia}
Wikipedia.
\newblock Polynomial kernel --- {W}ikipedia{,} the free encyclopedia, 2016.
\newblock URL: \url{https://en.wikipedia.org/wiki/Polynomial_kernel}.

\bibitem{Zhang}
Yuchen Zhang, Jason Lee, and Michael Jordan.
\newblock $\ell_1$ networks are improperly learnable in polynomial-time.
\newblock In {\em ICML}, 2016.

\end{thebibliography}
\bibliographystyle{plain}

\end{document}